\documentclass[final,12pt]{alt2024} % Include author names

\title[Tight Query Complexity Bounds for DRAs in Private Learning]{On the Query Complexity of Training Data Reconstruction in Private Learning}
\usepackage{times}
\usepackage{comment}
\usepackage{nicefrac}
\usepackage{jmlr_macros}
\usepackage{algorithmic}
\usepackage{algorithm}
\usepackage{graphicx}
\graphicspath{ {./images/} }
\usepackage{caption}
\usepackage{subcaption}

\altauthor{%
 \Name{Prateeti Mukherjee} \Email{t-mukherjeep@microsoft.com}\\
 \addr Microsoft Research
 \AND
 \Name{Satya Lokam} \Email{satya.lokam@microsoft.com}\\
 \addr Microsoft Research%
}

\begin{document}

\maketitle

\begin{abstract}%
We analyze the number of queries that a whitebox adversary needs to make to a private learner in order to reconstruct its training data. For $(\epsilon, \delta)$ DP learners with training data drawn from any arbitrary compact metric space, we provide the \emph{first known lower bounds on the adversary's query complexity} as a function of the learner's privacy parameters. \emph{Our results are minimax optimal for every $\epsilon \geq 0, \delta \in [0, 1]$, covering both $\epsilon$-DP and $(0, \delta)$ DP as corollaries}. Beyond this, we obtain query complexity lower bounds for $(\alpha, \epsilon)$ R\'enyi DP learners that are valid for any $\alpha > 1, \epsilon \geq 0$. Finally, we analyze data reconstruction attacks on locally compact metric spaces via the framework of Metric DP, a generalization of DP that accounts for the underlying metric structure of the data. In this setting, we provide the first known analysis of data reconstruction in unbounded, high dimensional spaces and obtain query complexity lower bounds that are nearly tight modulo logarithmic factors. 
\end{abstract}
%\begin{keywords}%
%  Differential Privacy, Training Data Reconstruction, Hypothesis Testing%
%\end{keywords}
\section{Introduction}
\label{sec:introduction}
Machine Learning has become increasingly pervasive, finding applications in multiple real world, risk-sensitive workflows. The fascinating potential of machine learning is perhaps most apparent in recent times, with the development of large-scale deep learning algorithms such as Large Language Models (\cite{gpt-3}) and Diffusion Models (\cite{dall-e-model}). These powerful algorithms consume enormous amounts of data to derive meaningful patterns for improved empirical performance. Theoretical analyses \citep{Bresler2020, eldan-memorization, ohadshamir-memorization}, however, reveal that Neural Networks are capable of memorizing the underlying training data, rendering them vulnerable to information leakage and adversarial attacks that can infer sensitive attributes of the training set, or worse, reconstruct one or more training samples entirely. Differential privacy was introduced in an attempt to formalize the notion of privacy protection while deriving meaningful conclusions from statistical information, and has since evolved into the \emph{de-facto} notion of privacy in machine learning. This has inspired a plethora of research in the development of private counterparts to popular machine learning \citep{Kamalika2008, Jain2011, Kamalika2009, Alabi2020} and deep learning \citep{Abadi2016,Jie2022,Song2013} algorithms. However, the promise of differential privacy is often difficult to interpret, and a rigorous quantification of \emph{``how much data privacy does an $(\e,\de)$-DP guarantee actually confer?"} is important to develop sufficiently private algorithms.

\noindent Analysis in this setting involves measuring the level of protection conferred by private learners against information leakage, a quantity that is largely specified by an adversary's success rate in performing a particular class of attacks. For instance, \cite{Yeom2018} demonstrate that, for the class of Membership Inference Adversaries (MIAs), where the objective is to infer the membership status of a target individual in the training set of a private learner, an adversary cannot perform much better than random guessing if the learner is $\e$-DP with $\e\leq0.4$. On the flip side, a more recent work of \cite{Humphries2020} suggests that, under minimal modeling assumptions, if $\e\geq2$, one can design adversaries that can correctly predict the membership status of a target individual through the output of the $\e$-DP learner with probability $\geq88\%$. This is concerning, since, most practical deployments of $\e$-DP algorithms consider $\epsilon\in[2,4]$ to be sufficiently private (\cite{smartnoise-pdf,apple-privacy}). From the perspective of an individual supplying their sensitive information to a learning algorithm, under the impression that their participation is kept secret, such large lower bounds on the success probability of an attack implies little privacy protection. Therefore, for effective mitigation of privacy risks in this context, $\e$ must be sufficiently small. However, prior work of \cite{Tramer2021} has already established that private learning in the small $\e$ regime yields models that perform significantly worse than their non-private counterparts, and find limited applicability in real-world deployments.

\noindent Membership status, however, may not be the security concern of interest in every scenario. For instance, an individual's online presence in a social network platform is generally public, and therefore may not qualify as sensitive information. However, if an adversary is able to \emph{reconstruct} private message transcripts or the individual's personal details, their privacy is at risk. In machine learning tasks, this would be an instance of a Training Data Reconstruction Attack (DRA), wherein, the adversary attempts to \emph{reconstruct} a target sample of the training dataset through the output of a learner. \cite{carlini2021extracting} have demonstrated that state of the art language models are susceptible to simple reconstruction adversaries that can perfectly reconstruct personally identifiable information and 128-bit UUIDs through repeated queries to the model, even when the sequences appear just once in the training set. This is particularly concerning in recent times, as we witness widespread adoption of powerful algorithms that necessitate the use of large scale, diverse training sets, procured from various sources. The results of \cite{Balle2022,carlini2021extracting,zhu2019deep} indicate that unauthorized recovery, and subsequent misuse, of an individual's sensitive information is a very tangible threat. Therefore, it is imperative that learning algorithms be scrutinized for their resilience to training data reconstruction attacks.

\noindent Privacy analysis in this setting has some nice properties. Intuitively, reconstruction must be harder than membership inference, since the latter involves recovering a single binary random variable, while the former may require reconstructing samples that are supported on some arbitrary domain beyond $\{0,1\}^d$ . This leads us to the question \emph{`can DP imply meaningful protection against DRAs for a larger range of $\e$, even when membership is compromised?'}. Empirically, prior work of \cite{Balle2022} has established that this hypothesis is indeed true. In this work, we answer this question by studying the theoretical underpinnings of a private algorithm's resilience to such reconstruction adversaries.

\noindent To facilitate a formal analysis, we must first identify the parameters that characterize \emph{protection} offered by a private learner. For $(\e,\de)$-DP learners, that there must be a dependence on the privacy parameters ($\e,\de$) is obvious. Furthermore, the empirical success of the reconstruction adversaries in \cite{carlini2021extracting} indicate that the number of interactions between the adversary and the learner, which we call the adversary's query complexity $n$, is crucial to the success of the attack. Finally, we note that privacy violations are largely \emph{context-dependent}. For instance, it is public information that Bill Gates is a billionaire; if an adversary is able to recover a rough approximation of his net worth with error of the order of $10^9$, it is very unlikely to be a violation of his privacy. However, if the adversary's reconstruction has an error of the order $10^2$, this level of precision suggests a very detailed and potentially invasive investigation into his financial affairs, thereby raising privacy concerns. Therefore, the level of protection offered by a private learner against reconstruction adversaries is also parameterized by the reconstruction error tolerance $\beta$, i.e., the precision to which a worst-case adversary is able to reconstruct a training sample, through the output of the private learner. From the learner's point of view, $\beta$ is the \emph{threshold of non-privacy} -- if an adversary is able to reconstruct any training sample with error $<\beta$, it counts as a privacy violation with respect to DRAs. Intuitively, this can be thought of as a \emph{smoothed} version of the notion of blatant non-privacy proposed in the seminal work of \cite{DinurNissim}. To this end, we seek answers to the following question: \emph{``What is the maximum number of queries that a private learner can safely answer, while avoiding highly accurate reconstruction of training samples by any adversary?''}

\subsection{Contributions}
\label{subsec:contributions}
Our work aims to present a comprehensive analysis of the effectiveness of training data reconstruction attacks against private learners. To this end, we establish \emph{tight non-asymptotic bounds} on the query complexity of \emph{any informed reconstruction adversary}, as a function of the privacy parameters of the learner and the system's purported tolerance to reconstruction attacks.  Our contributions are summarized as follows:

\paragraph{Reconstruction Attacks on $(\epsilon, \delta)$ Differentially Private Learners} Our first result analyzes data reconstruction attacks on $(\epsilon, \delta)$-DP learners whose training dataset is supported on some arbitrary compact metric space. In particular, we operate under a similar threat model to that of \cite{Balle2022, Guo2022} (see Section \ref{sec:problem-formulation} for a detailed overview) and tightly characterize the query complexity, i.e. the number of queries that the attacker needs to make, as a function of $\epsilon, \delta$ and the non-privacy threshold $\beta$. The result is informally stated below:
\begin{theorem}[DRAs on DP Learners (Informal)]
\label{thm:dra-dp-informal}
Let $\cM$ be an $(\epsilon, \delta)$-DP learning algorithm whose training data is supported on a compact metric space $(\cZ, \rho)$. Then, any data reconstruction adversary $\cA$ needs to make at least $\Omega(\tfrac{\ln(\nicefrac{\diam(\cZ)^2}{\beta^2})}{\ln(\tfrac{e^\epsilon + 1}{2(1-\delta)})})$ queries to $\cM$ in order to ensure a squared reconstruction error of at most $\beta^2$. Furthermore, the obtained query complexity is minimax optimal upto constant factors.
\end{theorem}
To the best of our knowledge, Theorem \ref{thm:dra-dp-informal} is the first known analysis of DRAs that holds for arbitrary compact metric spaces, and is minimax optimal for all $(\epsilon, \delta)$ regimes. In particular, it gives minimax optimal query complexities of $\Omega(\tfrac{\ln(\nicefrac{\diam(\cZ)^2}{\beta^2})}{\ln(\tfrac{e^\epsilon + 1}{2})})$ for $\epsilon$-DP and $\Omega(\tfrac{\ln(\nicefrac{\diam(\cZ)^2}{\beta^2})}{\ln\left(\nicefrac{1}{1-\delta}\right)})$ for $(0, \delta)$-DP learners.
\paragraph{Reconstruction Attacks on $(\alpha, \epsilon)$-R\'enyi DP Learners} Within the framework of Theorem \ref{thm:dra-dp-informal}, we analyze DRAs on $(\alpha, \epsilon)$-R\'enyi DP learners and obtain a query complexity lower bound of $\Omega(\tfrac{\ln(\nicefrac{\diam(\cZ)^2}{\beta^2})}{\ln(\frac{1}{(1 + e^{\gamma \epsilon})^{-1} + (1 + e^{\gamma \epsilon})^{\nicefrac{-1}{\gamma}}})})$ where $\gamma = 1 - \nicefrac{1}{\alpha}$. To the best of our knowledge, this is the first non-asymptotic analysis of data reconstruction on R\'enyi DP learners which holds for any $\alpha > 1$ and $\epsilon \geq 0$. This result also recovers the minimax optimal $\epsilon$-DP bound of Theorem \ref{thm:dra-dp-informal} as $\alpha \to \infty$. A key component of our proof is a novel total variation bound on the output distributions of R\'enyi DP learners, which could be of independent interest.  

\paragraph{Metric Differential Privacy and Reconstruction Attacks on Locally Compact Spaces} We now extend our analysis beyond standard differential privacy by considering the notion of \emph{Metric Differential Privacy} or \emph{Lipschitz Privacy} \citep{Chatzikokolakis2013, Koufogiannis2015, Koufogiannis2016, Imola2022, boedihardjo2022private}, a privacy formalism that generalizes differential privacy by accounting for the underlying metric structure of the data (beyond the Hamming metric). Operating under this broader notion of privacy, we analyze DRAs on $\epsilon$-metric DP learners, whose training data is supported on \emph{arbitrary locally compact metric spaces} (which may be unbounded) and obtain tight query complexity bounds. Our result is informally stated as follows:
\begin{theorem}[DRAs on Metric DP Learners (Informal)]
\label{thm:dra-mdp-informal}
Let $\cM$ be an $\epsilon$-metric DP learning algorithm whose training data is sampled from a locally compact metric space $(\cZ, \rho)$. Let $\Tilde{d}$ be the metric entropy of the unit ball in $(\cZ,\rho)$.  Then, any data reconstruction adversary $\cA$ needs to make at least $\Omega\left(\tfrac{\Tilde{d}}{\epsilon^2 \beta^2}\right)$ queries to $\cM$ in order to ensure a squared reconstruction error of at most $\beta^2$. The obtained query complexity is minimax optimal upto logarithmic factors.\end{theorem}
To the best of our knowledge, Theorem \ref{thm:dra-mdp-informal} is the first known analysis of data reconstruction attacks that directly applies to unbounded metric spaces (such as all of $\bR^d$), as well as the first result that aims to quantify the semantics of Metric Differential Privacy. To complement these results, we also demonstrate in Appendix \ref{app-sec:mdp-in-practice} that commonly used privacy-inducing mechanisms and learning algorithms such as the Gaussian Mechanism, Stochastic Gradient Descent, and Projected Noisy Stochastic Gradient Descent, naturally satisfy metric differential privacy with little to no algorithmic modifications.
%\paragraph{Empirical Evaluation} We demonstrate the practical utility of our query complexity bounds via empirical evaluations on standard benchmarks used by prior works \citep{Balle2022,Guo2022}. 

\noindent Our work analyzes data reconstruction attacks almost entirely from first principles, relying only on the basic definitions of differential privacy (and its variants), classical privacy-inducing mechanisms \citep{Warner-RR,mcsherry2007mechanism} and well-established connections between information divergences and hypothesis testing \citep{tsybakov2004introduction}. 

\section{Notation and Preliminaries}
\label{sec:notation-prelims}
Let $\cD$ be an arbitrary collection of secrets, with samples drawn from the metric space $(\cZ, \rho)$. We assume $(\cZ, \rho)$ is locally compact unless stated otherwise. Note that this property is satisfied by almost every possible data domain, including finite dimensional normed spaces, such as $\reals^{d}$. We denote the Hamming metric as $\rho_{H}$. In practical applications, $\cD$ is modelled as an aggregation of $N$ samples $\{z_{1}, \dots, z_{N}\}$ and the measure of dissimilarity for any two such collections $\cD,\cD'\in \cZ^n$ is given by $\rho(\cD,\cD') = \min_{\sigma\in S_{N}}\sum_{i=1}^{N}\rho(\cD_{i},\cD_{\sigma(i)}')$, where $S_N$ is the set of all permutations of $[N]$. Two datasets $\cD,\cD'$ are \emph{neighboring} if they differ in a single element, i.e., $\rho_H(\cD,\cD')\leq1$.  We use the $\Omega, O$ and $\Theta$ notation to characterize the dependence of our rates on $n,d$ the privacy parameters of the algorithms and the system's threshold of non-privacy, suppressing numerical constants. We use $\gtrsim$ and $\lesssim$ to denote $\geq$ and $\leq$ respectively, modulo numerical constants. Beyond this, we also make use of the packing number of compact spaces (or compact subsets)
\begin{definition}[Packing Number (\cite{Wainwright2019})] An $\eta$-packing of the set $\mathcal{Z}$ with respect to the metric $\rho$ is a set $\{z_{i}\}_{i\in[M]}$ such that for all distinct $v,v'\in[M]$, we have $\rho(z_{v},z_{v'})\geq\eta$. The $\eta$-packing number $M(\eta, \mathcal{Z}, \rho) := \sup\{M\in\mathbb{N}:\exists \text{ an }\eta\text{-packing } z_{1}, ..., z_{M} \text{ of } \mathcal{Z}\}$
\end{definition}
\subsection{Differential Privacy}
\label{subsec:dp-definitions}
\begin{definition}[Pure Differential Privacy (\cite{Dwork2017})]
\label{defn:pure-dp}
    For any $\e\geq0$, a randomized learner $\cM:\cZ^N\to\cH$ is $\e$-differentially private(DP) if, for every pair of neighboring datasets $\cD,\cD'$, the output probability distributions satisfy:
    \begin{equation}
        \forall T \subseteq \cH, \; \bP[\cM(\cD)\in T]\leq e^{\e}\bP[\cM(\cD')\in T].
    \end{equation}
    Or, equivalently, the max divergence $D_\infty$ of the output distributions satisfies:
    \begin{equation}
        D_{\infty} = \sup_{x\in \mathsf{supp}(Q)}\log\left(\frac{P(x)}{Q(x)}\right)\leq\e
    \end{equation}
    where $P,Q$ denote the output distributions $\cM(\cD),\cM(\cD')$, respectively.
    \end{definition}
%\prateeti{TODO : ADD RENYI DP DEFN AND HOW IT RECOVERS PURE DP}
%\prateeti{Kairouz extremal mech. paper:  In a nutshell, differential privacy ensures that an adversary should not be able to reliably infer an individual’s record in a database, even with unbounded computational
%power and access to every other record in the database}
Informally, DP bounds the change in the output distribution induced by a randomized algorithm when one of the input samples is replaced or removed. The constraint ensures that the output distributions are sufficiently \emph{statistically indistinguishable}, often characterized by a constant contraction in the output space in terms of some statistical divergence. When defined in terms of the R\'enyi Divergence, we obtain the following popular relaxation of differential privacy:
%Informally, differential privacy is a measure of stability of a randomized algorithm, which quantifies the change in the output distribution of the algorithm when one of the input samples is replaced with another (arbitrary) sample. The constraint of differential privacy bounds this change, ensuring that the output distributions are sufficiently \emph{statistically indistinguishable}. This is often characterized by a constant contraction in the output space of the randomized mechanism, measured in terms of some statistical divergence metric. When defined in terms of the R\'enyi Divergence, we obtain the following popular relaxation of differential privacy:
\begin{definition}[R\'enyi Differential Privacy (\cite{Mironov2017})]
    \label{defn:rdp}
    For any $\e\geq0, \alpha>1$, a randomized learner $\cM:\cZ^{N}\to\cH$ is $(\alpha,\e)$-R\'enyi differentially private(RDP) if, for every pair of neighboring datasets $\cD,\cD'$, the R\'enyi Divergence is bounded as follows:
    \begin{equation}
    \label{eqn:rdp-def}
        D_{\alpha}(P||Q) = \frac{1}{\alpha - 1}\log\bE_{x\sim Q}\left[\left(\frac{P(x)}{Q(x)}\right)^{\alpha}\right]\leq\e
    \end{equation}
    %\satya{Corrected by the big parenthesis before powering to $\alpha$.}
    where $P$ and $Q$ denote the output distributions $\cM(\cD)$ and $\cM(\cD')$, respectively.  %The case of $\alpha = 1$ and $\alpha = \infty$ are defined by taking the appropriate limits on the LHS of \eqref{eqn:rdp-def}
\end{definition} 
\begin{comment}
        \begin{equation}
        \expctn_{o\sim\learner(\dataset')}\left[\frac{\prob[\learner(\dataset) = o]}{\prob[\learner(\dataset') = o]}\right] \leq e^{(\alpha-1)\epsilon}
    \end{equation}
\end{comment}
For $\alpha \to \infty$, R\'enyi DP recovers the definition of $\e$-DP (\cite{Mironov2017}). Another popular relaxation of differential privacy, termed approximate differential privacy, measures this change in terms of $(\e,\de)$-indistinguishability.
\begin{definition}[Approximate Differential Privacy(\cite{dwork2006approx})] 
\label{defn:approx-dp}
For any $\e\geq0, \de\in[0,1]$, a randomized learner $\cM:\cZ^N\to\cH$ is $(\e,\de)$ differentially private if for every pair of neighboring datasets $\cD,\cD'$, the output probability distributions $\cM(\cD)$ and $\cM(\cD')$ are $(\e,\de)$ indistinguishable, i.e.,
\begin{equation}
        \forall T \subseteq \cH, \; e^{-\e}(\bP[\cM(\cD')\in T]-\de)\leq\bP[\cM(\cD)\in T]\leq e^{\e}\bP[\cM(\cD')\in T] + \de
    \end{equation}
\end{definition}
\subsection{Metric Differential Privacy}
%For example, text data is largely unstructured, and, depending on the representations, it might not always make sense to use the Hamming metric as a measure of \emph{closeness} between two documents.
While there exist several variants and relaxations of DP, the traditional notion is fundamentally restricted to the discrete setting. Specifically, the definition requires that the private inputs to the randomized mechanism must belong to a product space, where collections can break down into natural single elements to allow for \emph{neighboring datasets} to exist. We argue that this requirement is too stringent, particularly in situations where the sensitive information does not belong to a database at all, but is some arbitrary collection of secrets. In unstructured data domains, such as text, it is challenging to establish a natural definition of neighbors, but there does exist a notion of \emph{distinguishability} of representations. For instance, consider a topic classification problem where the author's identity is to be protected. In this case, two representations that are "similar in topic" must remain so in the output of a private mechanism, irrespective of the author. This notion of distance is not natural to the Hamming metric, and requires more sophisticated metric (e.g. Earth Mover's distance (\cite{text-processing})). In fact, general measures also do not break down into natural single elements, as is required by DP (\cite{boedihardjo2022private}). To facilitate privacy protection in such settings, we would need to incorporate the underlying metric structure of the input space of the private algorithm in the parameter that bounds the shift in the output probability distributions, when altering a single input. The following notion incorporates this desiderata, and extends the classical concept of DP to general metric spaces. 
\begin{definition}[Metric Differential Privacy (\cite{Chatzikokolakis2013})]\label{defn:mDP}
 Let $(\cZ,\rho)$ be a locally compact metric space and $\cH$ be measurable. For any $\e\geq0$, a randomized learner $\cM: \cZ^{N}\to\cH$ is $\e$-metric differentially private (mDP) if for every pair of inputs $z,z'\in\cZ$ and any measurable subset $T\subset\cH$,
    \begin{equation}
    \bP[\cM(z)\in T]\leq e^{\e\rho(z,z')}\bP[\cM(z')\in T]
        %\footnote{Note that DP is a special case of lipschitz-DP in the Hamming metric, for \(\rho_{H}(x,x')=1\).}
    \end{equation}
\end{definition}
We observe that metric differential privacy (often appearing under the pseudonym of Metric Privacy \citep{boedihardjo2022private} and Lipschitz Privacy \citep{Koufogiannis2016}) is a strict generalization of differential privacy. In fact, as we shall show in Appendix \ref{app-sec:mdp-in-practice}, the two definitions are equivalent when the inputs belong to a product space equipped with the Hamming metric. Beyond this, we also discuss in Appendix \ref{app-sec:mdp-in-practice} how popular privacy preserving mechanisms in the DP literature, such as the Gaussian mechanism, and learning algorithms, such as Projected Noisy SGD (\cite{Feldman2018}) satisfy metric DP with little to no modifications.
\section{Formalizing Data Reconstruction Attacks as a Privacy Game}
\label{sec:problem-formulation}
Our analysis begins by formalizing data reconstruction attacks as a \emph{privacy game} between the learner $\cM$ and the reconstruction adversary $\cA$, which proceeds as follows:\\

\noindent Let $\advData := \{z_{1}, ..., z_{N-1}\}$ be a fixed dataset, comprising of $N-1$ training samples, \emph{that is known to both the learner and the adversary.}

\noindent\textbf{Learner} chooses an arbitrary target sample $\challenge\in\cZ$, appends it to the training dataset $\cD=\advData\cup\{\challenge\}$, and outputs a sample $h$ drawn from its output distribution \footnote{By definition of differential privacy, $\cM$ must be a randomized algorithm, see \cite{complexityDP}} $\cM(\cD)$. 

\noindent\textbf{Adversary} makes $n$ queries to the model, i.e., she draws $n$ samples $h_1, \dots h_n \sim \cM(\cD)$ from the learner's output distribution, and generates a reconstruction $\cA(h_1, \dots, h_n)$ of the target sample $\challenge$. Since adversaries are generally resource bounded in almost all practical settings, we assume that the query complexity of $\cA$, i.e., the number of times that $\cA$ can query $\cM$ is finite.\\

\noindent A prototypical example of $\cM$ would be a logistic regression classifier trained with DP-SGD (with $h \sim \cM(\cD)$ representing the regressor/model weights). Examples of $\cA$ include the GLM attack of \cite{Balle2022}.

\noindent We note that such a privacy game formulation of reconstruction attacks is also the cornerstone of prior works such as \cite{Balle2022} and \cite{Guo2022} (when restricted to the case of $n=1$). In fact, analogous privacy game formulations for MIAs are highly predominant in a wide range of applications including (but not limited to):interpretations of the semantics of DP guarantees \citep{mahloujifar2022optimal,Humphries2020,Yeom2018} and devising auditing strategies for practical deployments of private learning algorithms \citep{salem2023sok, rezaAudit}\\

\noindent Equipped with the above formulation, it is natural to quantitatively analyze data reconstruction attacks via the machinery of two-player zero-sum games \citep{von1947theory}. To this end, an intuitive choice of the utility/payoff function of the learner is given by \footnote{More generally, one can set $u_{\cM} = \bE_{h_1,\dots, h_n \sim \cM(\advData \cup \{z\})}[g(\rho(\cA(h_1, \dots, h_n), z))]$ for some non-negative increasing differentiable function $g$. We choose $g(t) = t^2$ for clarity but our proof techniques extend to arbitrary $g$} $u_{\cM}(\cA, z,n) = \bE_{h_1,\dots, h_n \sim \cM(\advData \cup \{z\})}[\rho(\cA(h_1, \dots, h_n), z)^2]$ and the corresponding value function of the game is given by
\begin{equation}
\label{eqn:dra-value-fn} 
\cV(\cM,n) = \inf_{\cA} \sup_{z \in \cZ} \bE_{h_1,\dots, h_n \sim \cM(\advData \cup \{z\})}[\rho(\cA(h_1, \dots, h_n), z)^2]
\end{equation}
where the infimum is taken over all reconstruction adversaries. Note that, given any error tolerance parameter $\beta \geq 0$, $\cV(\cM, n) \geq \beta^2$ ensures that no data reconstruction adversary $\cA$ that is allowed to make at most $n$ queries to $\cM$ can uniformly obtain an expected squared reconstruction error less than $\beta^2$ on every target point $z \in \cZ$. To this end, we say that an adversary has attained the threshold of non-privacy $\beta$, if for any possible choice of the target point $z$, they are able to reconstruct $z$ upto an expected squared error of at most $\beta^2$. The remainder of our work aims to derive tight lower bounds for $\cV(\cM, n)$ as a function of $n$ and the privacy parameters of the learner $\cM$. This in turn can be used to derive \emph{tight query-complexity lower bounds for informed adversaries}, i.e., quantify the number of times an adversary needs to query the model in order to reach a given threshold of non-privacy. 
%\begin{comment}
\begin{remark}
This instantiation of a powerful white-box reconstruction adversary is borrowed from prior work of \cite{Balle2022}. Although the modeling may seem very stylized, investigating provable mitigations conferred by $(\e,\de)$-DP learners against such worst-case adversaries is a useful exercise -- theoretical lower bounds on the adversary's error suggest that such algorithms implicitly protect against reconstruction of training samples by less powerful, more realistic adversaries. %We also note that this is the exact threat model (for the special case of $n=1$) studied in \cite{Guo2022}, thereby facilitating a comparison between our lower bound results of Corollary \ref{cor:dra-renyi-dp} and that of \cite{Guo2022}, when specializing our results to the Euclidean MSE of reconstruction for $n=1,\alpha=2$.\prateeti{As illustrated in Table \ref{table}}
\end{remark}    
\section{Data Reconstruction Attacks against Differentially Private Learners}
\label{sec:dra-dp-bounds}
Our first result is a query complexity lower bound for data reconstruction attacks on $(\epsilon, \delta)$ DP learners whose training dataset is supported on a compact metric space $(\cZ, \rho)$. The proof of this result is presented in Appendix \ref{prf:dra-dp-lowerbound-proof}.
\begin{theorem}[Query Complexity Lower Bounds for $(\epsilon, \delta)$-DP Learners]
\label{thm:dra-dp-lowerbound}
Let $(\cZ,\rho)$ be any compact metric space. Consider any arbitrary $\epsilon \geq 0$ and $\delta \in [0,1]$. Then, for any target sample $z \in \cZ$ in the training dataset of an $(\epsilon, \delta)$ differentially private learner $\cM$, and any data reconstruction adversary $\cA$ that queries the learner $n$ times, the following holds:
\small
\begin{align*}
    \cV(\cM, n) \gtrsim \ \diam(\cZ)^{2} \left(\frac{2(1-\delta)}{e^\epsilon + 1}\right)^n
\end{align*} \normalsize
where $\diam(\cZ) = \max_{z, z^{\prime} \in \cZ} \rho(z, z^{\prime})$.

\noindent Consequently, for any fixed non-privacy threshold $\beta \geq 0$, the number of times $n$ that $\cA$ needs to query $\cM$ in order to always ensure an expected squared reconstruction error of at most $\beta^2$, is lower bounded as $n\geq \Omega\left(\tfrac{\ln(\nicefrac{\diam(\cZ)^2}{\beta^2})}{\ln\left(\tfrac{e^\epsilon + 1}{2(1-\delta)}\right)}\right)$
%\noindent Consequently, for any fixed tolerance parameter $\beta \geq 0$, the number of times $n$ that $\cA$ needs to query $\cM$ in order to always ensure an expected squared reconstruction error of at most $\beta^2$, is lower bounded as follows:
%\begin{align*}
%    n \geq \Omega\left(\tfrac{\ln(\nicefrac{\diam(\cZ)^2}{\beta^2})}{\ln\left(\tfrac{e^\epsilon + 1}{2(1-\delta)}\right)}\right)
%\end{align*}
\end{theorem}
We now demonstrate the minimax optimality of the above lower bound by deriving a matching upper bound for the two point metric space $\cZ = \{z_1, z_2 \}$. The proof of this result is presented in Appendix \ref{prf:dra-dp-upperbound-proof}.
\begin{theorem}[Upper Bound for $(\epsilon, \delta)$-DP Learners]
\label{thm:dra-dp-upperbound} 
Let $(\cZ, \rho)$ be a two-point metric space, i.e., $\cZ = \{ z_1, z_2 \}$. Furthermore, let $z \in \cZ$ and $\beta \geq 0$ be arbitrary. Then, there exists an $(\epsilon, \delta)$ differentially private learner $\cM$ whose training dataset contains $z$, and a reconstruction adversary $\cA$ which makes $\Theta\left(\tfrac{\ln(\nicefrac{\diam(\cZ)^2}{\beta^2})}{\ln\left(\tfrac{e^\epsilon + 1}{2(1-\delta)}\right)}\right)$ queries to $\cM$ and achieves an expected squared error of $\beta^2$.
\end{theorem}
Our upper bound construction is an \emph{$(\epsilon, \delta)$ DP variant} of the seminal randomized response algorithm \citep{Warner-RR}. This choice is motivated by the fact that the $(\epsilon, \delta)$  randomized mechanism $\cM$ used in our proof is known to be \emph{complete among the class of $(\epsilon, \delta)$ DP algorithms}, i.e., any $(\epsilon, \delta)$ DP algorithm can be represented as a composition of $\cM$ with some (possibly) randomized algorithm $T$ (see \cite{kairouz2015composition}, \cite{murtaghvadhan} Lemma 3.2 and \cite{bunsteinke} ) We also highlight that our upper bound construction is general enough to be directly extendable to $k$-point metric spaces $\cZ = \{ z_1, \dots, z_k \}$ and can potentially be extended to arbitrary compact metric spaces $\cZ$ via covering arguments. 

\paragraph{Applicability and Minimax Optimality} We note that the lower bound of Theorem \ref{thm:dra-dp-lowerbound} holds for any compact metric space $(\cZ, \rho)$. Furthermore, Theorem \ref{thm:dra-dp-lowerbound} and Theorem \ref{thm:dra-dp-upperbound} are applicable for every $\epsilon \geq 0, \delta \in [0, 1]$. Thus, our obtained query complexity guarantee is \emph{minimax optimal in all regimes}, which, to the best of our knowledge, is the first such result of its kind. To this end, \emph{our results cover both pure DP (or $\epsilon$-DP) and $(0, \delta)$ DP}. In particular, Theorem \ref{thm:dra-dp-lowerbound} implies a query complexity lower bound of $\Omega(\tfrac{\ln(\nicefrac{\diam(\cZ)^2}{\beta^2})}{\ln\left(\tfrac{e^\epsilon + 1}{2}\right)})$ for $\epsilon$-DP and $\Omega(\tfrac{\ln(\nicefrac{\diam(\cZ)^2}{\beta^2})}{\ln\left(\nicefrac{1}{1-\delta}\right)})$ for $(0, \delta)$ DP. Both these guarantees are minimax optimal as per Theorem \ref{thm:dra-dp-upperbound}. 

%\paragraph{Implications for Private Learners} \prateeti{Talk about threshold of non-privacy after refining the maintext. This is a smoothed analysis of Dinur Nissim}

\paragraph{Bounded Domain Assumption} Although the bound in Theorem \ref{thm:dra-dp-lowerbound} is minimax optimal, and to the best of our knowledge, the most general such result of its kind, we note that it requires $\diam(\cZ) < \infty$. We emphasize that this assumption is not specific to our analysis, and is in agreement with several prior works that investigate the protection offered by traditional DP (including variants and relaxations) algorithms \citep{Yeom2018,Tramer2020,Guo2022,Balle2022}. We conjecture that this assumption cannot be avoided in the analysis of traditional DP algorithms without making stringent modelling assumptions on the either the private learner, the data distribution, or the adversary's reconstruction algorithm. This is due to the fact that the standard notion of differential privacy (as defined in Definitions \ref{defn:pure-dp} and \ref{defn:approx-dp}) is oblivious to the underlying metric structure of the input space, thereby yielding worst-case behavior. In Section \ref{sec:dra-mdp-bounds}, we remove this boundedness assumption and extend the analysis to arbitrary metric spaces, presenting the first query complexity bounds for data reconstruction attacks against metric DP learners.
\subsection{Comparison to Prior Work}
\label{subsec:kamalika-comparison}
To the best of our knowledge, the result closest to Theorem \ref{thm:dra-dp-lowerbound} in prior works is Theorem 1 of \cite{Guo2022}, which considers data reconstruction attacks under the same threat model as ours (see Section \ref{sec:problem-formulation}), but is restricted to $(2, \epsilon)$-R\'enyi DP and $n=1$. Furthermore, it requires the data domain $\cZ$ to be a compact subset of $\bR^d$. For the sake of completeness, we restate their result as follows:
\begin{theorem}[Theorem 1 \cite{Guo2022}]
\label{thm:guo-2022}
For a target sample $z\in\cZ\subseteq\reals^d$ in the training set of a $(2,\e)$-R\'enyi DP learner, the MSE of a reconstruction attack that outputs an unbiased estimate $\hat{z}$ of $z$ upon observing $h\leftarrow\cM(\advData\cup\{z\})$ follows $\bE_h[\|\hat{z(h)}-z\|^{2}_{2}] \geq \frac{\sum_{i=1}^{d}\mathsf{diam}_{i}(\cZ)^{2}}{4(e^{\e}-1)}$. The i-th coordinate-wise diameter is given by $ \mathsf{diam}_{i}(\cZ):=\underset{z,z'\in\cZ, z_{j} = z'_{j}\forall j\neq i}\sup|z_{i}-z_{i}'|$.
\end{theorem}
\paragraph{Invalid Lower Bound} As we shall prove in Appendix \ref{app-subsec:guo-comparison}, \emph{the lower bound implied by Theorem 1 of \cite{Guo2022} is invalid for any $\epsilon < \ln(1 + \nicefrac{d}{4})$}, i.e., it violates trivial upper bounds implicit in the notion of the squared distance in bounded domains (this can also be intuitively seen by observing that the RHS diverges to infinity exponentially fast as $\epsilon \to 0$). To this end, the above lower bound result is vacuous for a surprisingly large range of $\epsilon$, particularly in high dimensional settings. Even for elementary use-cases like the MNIST dataset \citep{Lecun1998}, where $d=784$, the lower bound of \cite{Guo2022} is vacuous for any $\epsilon < 5.28$. We recall that practical deployments of DP algorithms typically use $\epsilon \in [2, 4]$ \citep{apple-privacy}. On the contrary, our results in Theorem \ref{thm:dra-dp-lowerbound} and \ref{thm:dra-dp-upperbound} are minimax optimal in all regimes $\epsilon \geq 0, \delta \in [0,1]$. Beyond this, we also establish a query complexity lower bound for $(\alpha, \epsilon)$ R\'enyi DP in Corollary \ref{cor:dra-renyi-dp}, which holds for any $\alpha > 1, \epsilon \geq 0$ (as opposed to the specific case of $\alpha = 2$, as in Theorem 1 of \cite{Guo2022})

\paragraph{Restriction to compact subsets of $\bR^d$} We highlight that Theorem 1 of \cite{Guo2022} is restricted only to compact subsets of $\bR^d$ whereas our lower bound applies to any compact metric space. In fact, the validity of \cite{Guo2022} Theorem 1 is dependent on the coordinate wise diameter $\diam_i(\cZ)$ being well defined. We note that there exist several applications where this condition may not be satisfied. For instance, let $\cZ$ be a collection of strings of varying lengths (where the maximum length is finite), equipped with the edit distance metric. The coordinate-wise edit distance is the maximum difference between corresponding characters of any two strings of the same length, which is not well-defined when $\cZ$ is composed of variable-length strings. Thus, their lower bound does not apply to this setting. On the contrary, our result in Theorem \ref{thm:dra-dp-lowerbound} is still valid, since $\mathsf{diam}(\cZ)$ is well-defined and represents the maximum dissimilarity between \emph{any} two strings in $\cZ$.
\subsection{Data Reconstruction Attacks on R\'enyi DP Learners}
\label{subsec:dra-rdp-bounds}
We now consider reconstruction attacks on $(\alpha, \epsilon)$R\'enyi DP learners (for $\alpha > 1, \epsilon \geq 0$). To this end, we derive the following query complexity lower bound by extending the proof technique of Theorem \ref{thm:dra-dp-upperbound} to the notion of R\'enyi DP. A key step in the proof involves deriving a novel Total Variation bound between the output distributions of R\'enyi DP Learners (Lemma \ref{lem:approx-rdp-tv} in Section \ref{sec:proof-sketch}), which could be of independent interest.

\begin{corollary}[Query Complexity Lower Bounds for $(\alpha,\epsilon)$-R\'enyi DP]
\label{cor:dra-renyi-dp} 
Let $(\cZ,\rho)$ be any compact metric space. Consider any arbitrary $\alpha > 1$ and $\epsilon \geq 0$. Then, for any target sample $z \in \cZ$ in the training dataset of an $(\alpha, \epsilon)$-R\'enyi DP learner $\cM$, and any data reconstruction adversary $\cA$ that queries the learner $n$ times, the following holds:
\small\begin{align*}
    \cV(\cM, n) \gtrsim \ \diam(\cZ)^{2} \left[\frac{1}{e^{\gamma \epsilon} + 1} + \frac{1}{\left(e^{\gamma \epsilon} + 1\right)^{\nicefrac{1}{\gamma}}}\right]^n
\end{align*}\normalsize
where $\gamma = 1 - \nicefrac{1}{\alpha} \in (0, 1)$ and $\diam(\cZ) = \max_{z, z^{\prime} \in \cZ} \rho(z, z^{\prime})$. 

\noindent Consequently, for any fixed non-privacy threshold $\beta \geq 0$, the number of times $n$ that $\cA$ needs to query $\cM$ in order to always ensure an expected squared reconstruction error of at most $\beta^2$, is lower bounded as follows: \small
\begin{align*}
    n \geq \Omega\left(\tfrac{\ln(\nicefrac{\diam(\cZ)^2}{\beta^2})}{\ln\left(\frac{1}{(1 + e^{\gamma \epsilon})^{-1} + (1 + e^{\gamma \epsilon})^{\nicefrac{-1}{\gamma}}}\right)}\right)
\end{align*} \normalsize
\end{corollary}
We note that Corollary \ref{cor:dra-renyi-dp} applies for any $\alpha > 1, \epsilon \geq 0$, in contrast to prior work \citep{Guo2022} which only considers the case of $\alpha = 2$. Furthermore, as $\alpha \to \infty$, Corollary \ref{cor:dra-renyi-dp} exactly recovers the minimax optimal query complexity lower bound for $\epsilon$-DP learners as implied by Theorem \ref{thm:dra-dp-lowerbound}.
\section{Data Reconstruction Attacks on Metric DP Learners}
\label{sec:dra-mdp-bounds}
We now present our results for data reconstruction attacks on locally compact linear spaces by considering metric DP learners. In this setting, we obtain the following lower bound, which we prove in Appendix \ref{prf:dra-mdp-lowerbound-proofs}
\begin{theorem}[Lower Bounds for $\epsilon$ Metric DP Learners]
\label{thm:dra-mdp-lowerbound}
Let $(\cZ, \rho)$ be any locally compact metric space and let $\Tilde{d} = \log M(\mathbb{B}(\cZ), \rho, \nicefrac{1}{2})$, where $\mathbb{B}(\cZ)$ is the unit ball in $\cZ$. For any target sample $z \in \cZ$ in the training dataset of an $\epsilon$ metric differentially private learner $\cM$, and any data reconstruction adversary $\cA$ that queries the learner $n$ times, $\cV(\cM, n) \geq \ \Omega\left(\frac{\Tilde{d}}{n \epsilon^2}\right)$. Consequently, for any fixed non-privacy threshold $\beta \geq 0$, $\cA$ needs to query $\cM$ at least $\Omega(\nicefrac{\Tilde{d}}{\epsilon^2 \beta^2})$ times in order to always ensure an expected squared reconstruction error of at most $\beta^2$.
\end{theorem}
We note that for $(\bR^d, \|\cdot\|)$, (or more generally, any finite dimensional normed space), $\Tilde{d} = \Theta(d)$ and thus, the query complexity lower bound becomes $\Omega(\nicefrac{d}{\epsilon^2 \beta^2})$. We complement this with an upper bound that is tight modulo logarithmic factors. The proof is presented in Appendix \ref{prf:dra-mdp-upperbound-proofs}
\begin{theorem}[Upper Bound for $\epsilon$ Metric DP Learners]
\label{thm:dra-mdp-upperbound}
Let $\cM$ be an $\epsilon$ metric differentially private learner whose training datapoints are elements of the normed space $(\bR^d, \| \cdot \|_2)$. Let $z \in \bR^d$ be any arbitrary sample in the training dataset of $\cM$ and let $\beta \geq 0$ be arbitrary. Then, there exists a reconstruction adversary $\cA$ which makes $\Theta(\tfrac{d \log^2(d)}{\epsilon^2 \beta^2})$ queries to $\cM$ in order to achieve an expected squared error of at most $\beta^2$.
\end{theorem}
Unlike prior works on data reconstruction that require some boundedness assumption on the data domain, \citep{Yeom2018, Tramer2020, Guo2022, stock2022defending}, Theorem \ref{thm:dra-mdp-lowerbound} and Theorem \ref{thm:dra-mdp-upperbound}, are, to the best of our knowledge, the first near-optimal analysis of training data reconstruction for unbounded metric spaces
\begin{comment}
\paragraph{1. First (tight) result on data reconstruction in general metric spaces} To the best of our knowledge, this is the first result on training data reconstruction attacks, when the training samples belong to general, locally compact metric spaces. To accommodate for such generality, we employ the notion of metric differential privacy. This allows us to express the complexity of the reconstruction task in terms of the multi-scale geometry of the input metric space to the private learner. Note that, unlike prior work on attack lower bounds for differential privacy (), our approach in the mDP setting does not require any assumptions on the boundedness of $\cZ$.
\end{comment}
\paragraph{Dimension dependence} The nearly-matching upper and lower bounds of Theorem \ref{thm:dra-mdp-lowerbound} and Theorem \ref{thm:dra-mdp-upperbound} captures the inherent hardness of the data reconstruction problem in high dimensions, and thus provides solid theoretical grounding to the observations of \cite{Balle2022} which shows that the performance of data reconstruction adversaries scales inversely with dimension. The intuition may be formalized as follows: let $\cZ$ be a finite dimensional vector space. Any target sample $z\in\cZ$ has $d$-degrees of freedom. From the lens of an adversary, the complexity of reconstructing the target sample must grow with $d$. Therefore, the number of queries the adversary needs to make to the learner to achieve good quality reconstruction must scale proportional to $d$, which is accurately captured in Theorem \ref{thm:dra-mdp-lowerbound}. 
\section{Proof Sketch}
\label{sec:proof-sketch}
Our analysis hinges on the following insight: Suppose the target sample $z$ is chosen from some finite set $S = \{z_1, \dots, z_k\}$ which is apriori known to the adversary. The task of the adversary then reduces to inferring (or testing) which one of $z_1, \dots, z_k$ is actually the target sample $z$, by using $n$ samples from $\cM(\advData \cup \{z\})$. Intuitively, this task of testing from samples cannot be harder than the original data reconstruction problem (wherein $z \in \cZ$ can be arbitrary). Such \emph{testing reductions} have a rich history in theoretical CS \citep{goldreich1998property}, learning theory \citep{kearns1994introduction} and statistics \citep{tsybakov2004introduction}. We complement this insight with the \emph{statistical indistinguishability} interpretation of differential privacy and its relaxations, i.e., for any private learner $\cM$, the output distributions $\cM(\advData \cup \{z_i\})$ and $\cM(\advData \cup \{z_j\})$ must be sufficiently close, which in turn imposes fundamental limits (quantified via information divergences) on the accuracy of the adversary's testing problem described above. 
\subsection{Proof of the Reconstruction Lower Bounds}
\label{subsec:lb-sketch}
We sketch the proof of Theorem \ref{thm:dra-dp-lowerbound} and Corollary \ref{cor:dra-renyi-dp}. For any $z \in \cZ$, let $P_z$ denote the output distribution of $\cM(\advData \cup \{z\})$ and let $P^n_z$ to be the associated product distribution. Recall that $\cV(\cM, n) = \inf_{\cA} \sup_{z \in \cZ} \bE_{h_{1:n} \sim P^{n}_z}[\rho(\cA(h_{1:n}), z)^2]$. Moreover, let $\Delta = \nicefrac{\diam(\cZ)}{2}$ and let $z_1, z_2$ be the two farthest points in $\cZ$, i.e., $\rho(z_1, z_2) = 2\Delta$. In accordance with the above discussion, we reduce the general reconstruction problem to the case when $z \in S = \{z_1, z_2\}$ by replacing the supremum in the definition of $\cV$ with an average over $S$ and applying Markov's inequality to obtain:
\begin{align*}
    \cV(\cM, n) &\geq\tfrac{\Delta^2}{2} \inf_{\cA} \left[P^{n}_{z_1}[\rho(\cA(h_{1:n}), z_1) \geq \Delta] + P^{n}_{z_2}[\rho(\cA(h_{1:n}), z_2) \geq \Delta]\right] 
\end{align*}
Since $\rho(z_1, z_2) = 2 \Delta$, an application of triangle inequality shows that $\rho(\cA(h_{1:n}), z_1) \geq \Delta$ holds whenever $\rho(\cA(h_{1:n}), z_1) \geq \rho(\cA(h_{1:n}), z_2)$. Similarly, $\rho(\cA(h_{1:n}), z_2) \geq \Delta$ holds whenever $\rho(\cA(h_{1:n}), z_1) \leq \rho(\cA(h_{1:n}), z_2)$. It follows that, 
\begin{align*}
    \cV(\cM, n) &\geq\tfrac{\Delta^2}{2} \inf_{\cA} \left[P^{n}_{z_1}[\rho(\cA(h_{1:n}), z_1) \geq \rho(\cA(h_{1:n}), z_2)] + P^{n}_{z_2}[\rho(\cA(h_{1:n}), z_1) \leq \rho(\cA(h_{1:n}), z_2)]\right] 
\end{align*}
By the definition of $\TV$, one can see that for any adversary $\cA$, the sum of the two probabilities on the RHS can be uniformly lower bounded by $1 - \TV(P^{n}_{z_1}, P^{n}_{z_2})$. Subsequently, using the tensorization properties of $1 - TV$ and recalling the definition of $P^n_z$, we obtain the following lower bound
\begin{align}
\label{eq:dp-dra-joint-eq}
    \cV(\cM, n) &\geq\tfrac{\diam(\cZ)^2}{8} \left(1 - \TV(\cM(\advData \cup \{z_1\}), \cM(\advData \cup \{z_2\}))\right)^{n}
\end{align}
The proof of both Theorem \ref{thm:dra-dp-lowerbound} and Corollary \ref{cor:dra-renyi-dp} are now concluded by bounding the Total Variation between the two output distributions of the private learner $\cM$. To this end, we first consider Theorem \ref{thm:dra-dp-lowerbound} and derive the following $\TV$ upper bound for the output distributions of $(\epsilon, \delta)$ DP learners in Appendix \ref{prf:kasi-lemma-proof}.
\begin{lemma}[TV Bounds for $(\epsilon, \delta)$ DP Learners]
\label{lem:approx-dp-tv}
Let $\cM$ be any arbitrary $(\epsilon, \delta)$ DP learner and let $\cD, \cD^{\prime}$ be any two arbitrary neighbouring datasets (i.e. datasets that differ in only one record). Then, the following holds: \small
\begin{align*} 
    \mathsf{TV}(\cM(\cD), \cM(\cD^{\prime})) \leq 1 - \frac{2(1 - \delta)}{1 + e^{\epsilon}}
\end{align*} \normalsize
\end{lemma}
For Corollary \ref{cor:dra-renyi-dp}, we derive the following novel $\TV$ bound between the output distributions of $(\alpha, \epsilon)$ Renyi DP Learners, which may be of independent interest. We prove this in Appendix \ref{prf:lem-approx-rdp-tv}
\begin{lemma}[TV Bounds for $(\alpha, \epsilon)$ Renyi DP Learners]
\label{lem:approx-rdp-tv}
Let $\cM$ be any arbitrary $(\alpha, \epsilon)$ Renyi DP learner and let $\cD, \cD^{\prime}$ be any two arbitrary neighbouring datasets (i.e. datasets that differ in only one record). Then, the following holds: \small
\begin{align*}
    \mathsf{TV}(\cM(\cD), \cM(\cD^{\prime})) \leq 1 - \frac{1}{1 + e^{\tfrac{(\alpha - 1)\epsilon}{\alpha}}} - \frac{1}{(1 + e^{\tfrac{(\alpha - 1)\epsilon}{\alpha}})^{\tfrac{\alpha}{\alpha - 1}}}
\end{align*} \normalsize
\end{lemma}
The proof of Theorem \ref{thm:dra-dp-lowerbound} and Corollary \ref{cor:dra-renyi-dp} are now completed by respectively applying Lemma \ref{lem:approx-dp-tv} and Lemma \ref{lem:approx-rdp-tv} to equation \eqref{eq:dp-dra-joint-eq}. We also highlight that the proof of Theorem \ref{thm:dra-mdp-lowerbound} proceeds along very similar lines, except that it involves a more delicate choice of the testing set $S = \{z_1, \dots, z_M\}$ using the fact that the definition of Metric DP is sensitive to the metric structure of the input. 
\subsection{Upper Bound for $(\epsilon, \delta)$ DP : Proof Sketch of Theorem \ref{thm:dra-dp-upperbound}}
Since $\cZ = \{ z_1, z_2 \}$, $\diam(\cZ) = \rho(z_1, z_2)$. Without loss of generality, let $z = z_1$, i.e., $\cD = \advData \cup \{ z_1 \}$. As stated earlier, our construction of $\cM$ is a variant of the randomized response mechanism defined as follows:

\begin{enumerate}
        \item Flip a coin $C \sim \mathsf{Bernoulli}(\delta)$. If $C = 1$, then $\cM(\cD) := (1,z)$
    \item Otherwise, if $C=0$, then $\cM(\cD)$ is defined as follows:
    \begin{align*}
        \cM(\cD) :=
        \begin{cases}
    (0,z) \text{ with probability } \tfrac{e^\epsilon}{1+e^{\epsilon}}\\
    (0,\cZ\setminus\{z\}) \text{ with probability } \tfrac{1}{1+e^{\epsilon}}
    \end{cases}
    \end{align*}
\end{enumerate}

\noindent As we shall show in Appendix \ref{prf:dra-dp-upperbound-proof}, $\cM$ is $(\epsilon, \delta)$ differentially private. We now consider a reconstruction adversary $\cA$ which draws $h_1, \dots, h_n \iidsim \cM(\cD)$ as follows: If $C_i = 1$ for some $i \in [n]$, $\cA(h_{1:n}) := z_i$. Else, define $\cA$ as follows:
\begin{comment}
\vspace{-3mm}
\small \begin{align*}
        \cM(\cD) :=
        \begin{cases}
    (0,z) \text{ with probability } \tfrac{e^\epsilon}{1+e^{\epsilon}}\\
    (0,\cZ\setminus\{z\}) \text{ with probability } \tfrac{1}{1+e^{\epsilon}}
\end{cases}
\end{align*} \normalsize
As noted by you, $\cM$ remains $(\e,\de)$ DP. Now, given $h_i \sim \cM(\cD)$, we define $\cA$ as follows: If $\exists $

Next, inspired by the Reviewer's suggestion, we construct an adversary that draws $n$ samples $h_{1:n}\sim\cM(\cD)$ and operates in two phases:
\end{comment}
\begin{equation*} 
\begin{aligned}[c]
        b_i :=
        \begin{cases}
    \sim \mathsf{Bern}(1-e^{-\e}) \text{ if } h_i = (0,z)\\
    0 \text{ if } h_i = (0,\cZ\setminus\{z\})
\end{cases}
\end{aligned}
\begin{aligned}[c]
        \cA(h_{1:n}) :=
        \begin{cases}
    z \text{ if } \wedge_{i=1}^{n} b_i = 1\\
    \cZ\setminus\{z\} \text{ otherwise } \\
\end{cases}
\end{aligned}
\end{equation*}
where $\wedge$ denotes bitwise OR. Conditioned on the event $\{ C_i = 0 \ \forall i \in [n] \}$, $b_i \iidsim \mathsf{Bern}(1-p)$ where $p = \tfrac{1}{e^\epsilon + 1} + e^{-\epsilon} \cdot \tfrac{e^\epsilon }{e^{\epsilon}+1} = \tfrac{2}{e^{\epsilon}+ {1}}$. Note that $\cA(h_{1:n}) = \cZ \setminus \{z\}$ iff $C_i = b_i = 0 \ \forall i \in [n]$. Since this event occurs with probability $(\tfrac{2(1-\delta)}{e^{\epsilon} + 1})^n$, it follows that
\begin{align*}
    \bE_{h_1, \dots, h_n \iidsim \cM(\cD)}[\rho(\cA(h_1, \dots, h_n), z)^2] = \left(\frac{2(1 - \delta)}{e^\epsilon + 1}\right)^{n} \diam(\cZ)^2  
\end{align*}

\noindent To set the RHS to be equal to $\beta^2$, it suffices to set $n = \Theta\left(\tfrac{\ln(\nicefrac{\diam(\cZ)^2}{\beta^2})}{\ln\left(\tfrac{e^\epsilon + 1}{2(1-\delta)}\right)}\right)$

\noindent We note that this adversary uses a randomized decision rule, which is consistent with our modeling assumptions in Section \ref{sec:problem-formulation} and of prior work (\citet{Guo2022,Balle2022}). In Appendix \ref{app-sec:dra-dp-proofs}, we construct a deterministic majority-based adversary and demonstrate that the lower bound is still tight $\ \forall \epsilon \geq 0.042,\de\in[0,1]$.

\subsection{Upper Bound for $\epsilon$ Metric DP: Proof Sketch of Theorem \ref{thm:dra-mdp-upperbound}}
Our upper bound construction for metric DP is based upon the exponential mechanism in $(\bR^d, \|\dot\|_2)$. In particular, for $\mu \in \bR^d$, let $\pi_{\mu}$ be a probability measure on $\bR^d$ whose density with respect to the Lebesgue measure $\mathsf{Leb}$ is given by,
\begin{align*}
    \frac{\mathsf{d} \pi_{\mu, \epsilon}}{\mathsf{d} \mathsf{Leb}}(x) \propto e^{-\epsilon \|x - \mu\|_2}
\end{align*}
For any $z \in \bR^d$, we set the output distribution of $\cM$ to be $\cM(\cD) = \cM(\advData \cup \{z\}) = \pi_{z, \epsilon}$. We verify in Appendix \ref{prf:dra-mdp-upperbound-proofs} that $\cM$ is $\epsilon$ metric DP w.r.t $\|\cdot\|$ since for any $\mu_1, \mu_2 \in \bR^d$, $\tfrac{\mathsf{d} \pi_{\mu_1, \epsilon}}{\mathsf{d} \pi_{\mu_2, \epsilon}} \leq e^{\epsilon \|\mu_1 - \mu_2\|_2}$ due to the triangle inequality. We consider a reconstruction adversary which draws $h_1, \dots, h_n \iidsim \cM(\cD)$ and computes $\cA(h_1, \dots, h_n) = \tfrac{1}{n} \sum_{i=1}^{n} h_i$. It is easy to see that $\bE_{h_1, \dots, h_n \iidsim \cM(\cD)}[\|\cA(h_1, \dots, h_n) - z\|^{2}_{2}] = \frac{\sigma^2}{n}$ where $\sigma^2 = \bE_{x \sim \pi_{z, \epsilon}}[\|x - z\|^{2}_{2}]$. Thus, bounding the reconstruction error reduces to sharply bounding $\sigma^2$. To this end, we use classical results in probability theory on the isoperimetric properties of the exponential distribution \citep{bobkov1997poincare, bobkov2003spectral} in $\bR^d$ as well the connections between moment control and isoperimetry \citep{aida1994moment,huang2021poincare, garg2020scalar} to sharply bound $\sigma^2$ as $\sigma^2 \lesssim \nicefrac{d \log^2(d)}{\epsilon^2}$. It follows that,
\begin{align*}
    \bE_{h_1, \dots, h_n \iidsim \cM(\cD)}[\|\cA(h_1, \dots, h_n) - z\|^{2}_{2}] \lesssim \frac{d \log^2(d)}{n \epsilon^2}
\end{align*}
To make the RHS at most $\beta^2$, it suffices to set $n = \Theta(\tfrac{d \log^2(d)}{\epsilon^2 \beta^2})$ 
%\section{Experiments}
%\label{sec:exps}

%This is where the content of your paper goes.
%\begin{itemize}
%  \item Use the \textbackslash documentclass[anon,12pt]\{alt2024\} option during submission process -- this automatically hides the author names listed under \textbackslash altauthor.
%
%  \item Submissions should NOT include author names or other identifying information in the main text or appendix. To the extent possible, you should avoid including directly identifying information in the text. You should still include all relevant references, discussion, and scientific content, even if this might provide significant hints as to the author identity. But you should generally refer to your own prior work in third person. Do not include acknowledgments in the submission. They can be added in the camera-ready version of accepted papers. 
%  
%  \item Use \textbackslash documentclass[final,12pt]\{alt2024\} only during camera-ready submission.
%\end{itemize}

% Acknowledgments---Will not appear in anonymized version
\begin{comment}
\acks{We thank a bunch of people and funding agency.}
\end{comment}

\bibliography{references}
\newpage
\appendix

% \crefalias{section}{appendix} % uncomment if you are using cleveref
\section{Additional Notation and Technical Lemmas}
\label{app-sec:tech-lemmas}

\begin{definition}[Total Variation Distance]
\label{def:tv}
Let $P_0$ and $P_1$ be two probability measures on a measurable space $(\cX,\cA)$. The Total Variation distance between $P_0,P_1$ is given by:
\begin{equation}
    \mathsf{TV}(P_0,P_1) = \sup_{T\in\cA}|P_0(T)-P_1(T)|
\end{equation}
\end{definition}

\begin{lemma}[Tensorization of Total Variation Affinity]
\label{lem:tv-tensorization}   
Let $P$ and $Q$ be two probability measures supported on a set $\cX$. For any $n \in \mathbb{N}$, let $P^n$ and $Q^n$ denote the respective product measures of $P$ and $Q$, supported on the set $\cX^n$. Then, the following holds:
\begin{align*}
    \TV(P^n, Q^n) \leq 1 - (1 - \TV(P, Q))^n
\end{align*}
\end{lemma}
\begin{proof}
Let $\Gamma(P, Q)$ denote the set of all couplings of $P$ and $Q$. We make use of the fact that $\TV(P, Q) = \min_{C \in \Gamma(P, Q)} \bP_{(x, y) \sim C}[x \neq y]$ \citep{levin2017markov}. To this end, let $C^*$ denote the $\TV$-optimal coupling of $P$ and $Q$, i.e., $C^* \in \Gamma(P, Q)$ such that $\TV(P, Q) = \bP_{(x, y) \sim C^*}[x \neq y]$. We now construct a coupling of $P^n$ and $Q^n$ as follows. Let $X = (x_1, x_2, \dots, x_n)$ and $Y = (y_1, y_2, \dots, y_n)$ be two random variables on $\cX^n$ such that $(x_i, y_i) \stackrel{i.i.d}{\sim} C^* \ \forall i \in [n]$. It is easy to see that $(X, Y)$ is a coupling of $P^n$ and $Q^n$. It follows that:
\begin{align*}
    \TV(P^n, Q^n) &\leq \bP[X \neq Y] \\
    &= 1 - \bP[X = Y] \\
    &= 1 - \prod_{i=1}^{n} \bP[x_i = y_i] \\
    &= 1 - \prod_{i=1}^{n} (1 - \bP[x_i \neq y_i]) \\
    &= 1 - (1 - \TV(P, Q))^n
\end{align*}
where the third step uses the co-ordinatewise i.i.d. structure of $(x_i, y_i)$ and the last step uses the fact that $(x_i, y_i)$ are sampled from the $\TV$-optimal coupling of $P$ and $Q$. Hence, $\TV(P^n, Q^n) \leq 1 - (1 - \TV(P, Q))^n$.
\end{proof}

\begin{definition}[Binary Hypotheses Testing]
\label{def:binary-hypothesis-test}
Let $P_0$ and $P_1$ be two probability measures and $X$ be an observation drawn from either $P_0$ or $P_1$. If the objective is to determine which distribution generated the observation $X$, it corresponds to the following statistical hypothesis testing problem:
\begin{equation}
\begin{aligned}
    & H_0 \mathsf{(null)}&: X \sim P_0\\
    &H_1 \mathsf{(alternate)}&: X\sim P_1
\end{aligned} 
\end{equation}
    A test $\Psi:X\to\{0,1\}$ indicates which hypothesis is true. 
\end{definition}
In a binary hypothesis test, the test $\Psi$ can make two kinds of errors: A type I error(false alarm) corresponds to the case when the null hypothesis is true but rejected, i.e., $\bP_{X\sim P_0}[\Psi(X)=1]$. A type II error(missed detection) is when the null hypothesis is false but retained, i.e., $\bP_{X\sim P_1}[\Psi(X)=0]$.

\begin{lemma}[Variational Representation of Total Variation(\cite{Cam1986AsymptoticMI})]
\label{lem:LeCam}
 For any distributions $P$ and $Q$ on a measurable space $(\cX,)$, we have:
 \begin{equation}
     \inf_{\Psi}(\bP_{X\sim P_0}[\Psi(X)=1] + \bP_{X\sim P_1}[\Psi(X)=0]) = 1 - \mathsf{TV}(P_0,P_1)
 \end{equation}
 where the $\inf$ is over all tests $\Psi$.
\end{lemma}

\begin{definition}[M-ary Hypothesis Test (\cite{Wainwright2019})]
\label{def:m-ary-hypothesis-test}
Let $P_1, ..., P_M$ be $M(\geq 2)$ probability distributions such that $P_j\ll P_k\forall j,k$ and $X$ be an observation drawn from any one of $P_1,...,P_M$. If the objective is to determine which distribution generated the observation $X$, it corresponds to an $M$-ary hypothesis testing problem, i.e., it is a generalization of Definition \ref{def:binary-hypothesis-test} to multiple hypotheses.    
\end{definition}

\begin{lemma}[Fano's Inequality for M-ary Hypotheses (\cite{Wainwright2019})]
\label{lemma:fanos}
Let $P_1, ..., P_M$ be $M(\geq 2)$ probability distributions such that $P_j\ll P_k\forall j,k$ and $X$ be an observation drawn from any one of $P_1,...,P_M$. Let $\Psi:X\to [M]$ be a test that indicates which distribution $X$ was drawn from, i.e., which of the $M$ hypotheses is true. Then, we have
\begin{equation}
    \inf_{\Psi}\max_{1\leq j\leq M} P_{j}[\Psi(X)\neq j]\geq 1-\frac{1/M^{2}\sum_{j,k=1}^{M}\mathsf{KL}(P_j,P_k) + \log 2}{\log(M-1)}
\end{equation}
where the $\inf$ is over all tests $\Psi$ with values in $[M]$
\end{lemma}
\begin{lemma}[Proposition 10 of \cite{Mironov2017}]
\label{lem:rdp-set-bounds}
Let $\cM$ be an $(\alpha, \epsilon)$ Renyi DP learner and let $\cD$ and $\cD^{\prime}$ be neighboring datasets (i.e. datasets differing in only one entry). Then, for any $S \subseteq \cH$, the following holds
\begin{align*}
    \bP[\cM(\cD) \in S] &\leq \left(e^{\epsilon} \bP[\cM(\cD^{\prime}) \in S]\right)^{\nicefrac{(\alpha - 1)}{\alpha}} \\
    \bP[\cM(\cD^{\prime}) \in S] &\leq \left(e^{\epsilon} \bP[\cM(\cD) \in S]\right)^{\nicefrac{(\alpha - 1)}{\alpha}}
\end{align*}
\end{lemma}
\begin{lemma}[KL Upper Bounds for Metric DP Learners]
\label{lem:kl-upper-bounds}
Let $\cM$ be an $\epsilon$ metric differentially private learner. Then, for any two datasets $\cD, \cD^{\prime}$, the KL divergence between the output distributions $\cM(\cD)$ and $\cM(\cD^{\prime})$ is upper bounded as follows:
\begin{align*}
    \mathsf{KL}[\cM(\cD) || \cM(\cD^{\prime})] \leq \frac{\epsilon^2 \rho(\cD, \cD^{\prime})^2}{2}
\end{align*}
\end{lemma}
\begin{proof}
This follows from Lemma 3.8 of  \citet{dwork2016concentrated} by replacing $\epsilon$ with $\epsilon \rho(\cD, \cD^{'})$. 
\end{proof}
\section{Analysis of Data Reconstruction for $(\epsilon, \delta)$ DP}
\label{app-sec:dra-dp-proofs}
\subsection{Proof of Lemma \ref{lem:approx-dp-tv}}
\label{prf:kasi-lemma-proof}
\begin{proof}
Consider any arbitrary $S \subseteq \cH$. By definition of total variation, the following holds
\begin{equation}
\label{eq:tv-basic-upperbound}    
\TV(\cM(\cD), \cM(\cD^{\prime}) \leq \bP[\cM(\cD) \in S] - \bP[\cM(\cD^{\prime}) \in S]
\end{equation}
Note that since $\cM$ is an $(\epsilon, \delta)$ DP mechanism and $\cD, \cD^{\prime}$ are neighboring datasets, the following constraints must be satisfied:
\begin{align}
\label{eq:tv-dp-constraints}
    \bP[\cM(\cD) \in S] &\leq e^{\epsilon} \bP[\cM(\cD^{\prime}) \in S] + \delta \nonumber \\
    \bP[\cM(\cD^{\prime}) \in S] &\leq e^{\epsilon} \bP[\cM(\cD) \in S] + \delta
\end{align}
From \eqref{eq:tv-basic-upperbound} and \eqref{eq:tv-dp-constraints}, it is easy to see that $\TV(\cM(\cD), \cM(\cD^{\prime})$ can be upper bounded by the solution to the following linear program:
\begin{align}
\label{eq:tv-dp-lp}
\ & \min x - y \nonumber \\
\text{subject to } & 0 \leq x \leq 1, \nonumber \\
\ & 0 \leq y \leq 1, \nonumber \\
\ & x \leq e^\epsilon y + \delta, \nonumber \\
\ & y \leq e^\epsilon x + \delta  
\end{align}
We now check that $x = 1 - \tfrac{1 - \delta}{e^{\epsilon} + 1}$ and y = $\tfrac{1 - \delta}{e^{\epsilon} + 1}$ lies in the constraint set of the linear program \eqref{eq:tv-dp-lp}. Clearly, since $\delta \in [0,1]$ and $\epsilon \geq 0$, $0 \leq x \leq 1$ and $0 \leq y \leq 1$ are trivially satisfied. We then note that,
\begin{align*}
    e^{\epsilon}y + \delta &= \frac{e^{\epsilon}(1-\delta)}{e^{\epsilon} + 1} + \delta \\
    &= (1 - \delta)(1 - \frac{1}{e^{\epsilon} + 1}) + \delta \\
    &= 1 - \frac{(1 - \delta)}{e^{\epsilon} + 1} = x
\end{align*}
i.e. the third constraint in \eqref{eq:tv-dp-lp} is satisfied as an equality. Finally,
\begin{align*}
    e^{\epsilon} x + \delta &= e^\epsilon \left[1 - \frac{1 - \delta}{e^\epsilon + 1}\right] + \delta \\
    &= e^\epsilon - \frac{e^\epsilon}{e^\epsilon + 1}(1-\delta) + \delta \\
    &= e^\epsilon - \left[1 - \frac{1}{e^\epsilon + 1}\right](1 - \delta) + \delta \\
    &= e^{\epsilon} + 2\delta - 1 + \frac{1 - \delta}{e^{\epsilon} + 1} \\
    &= e^{\epsilon} + 2\delta - 1 + y \geq y
\end{align*}
where the last inequality follows from the fact that $\epsilon \geq 0$ and $\delta \geq 0$. Thus, the fourth constraint in the linear program \eqref{eq:tv-dp-lp} is also satisfied. From \eqref{eq:tv-basic-upperbound}, \eqref{eq:tv-dp-constraints} an \eqref{eq:tv-dp-lp}, we conclude that,
\begin{align*}
    \TV(\cM(\cD), \cM(\cD^{\prime})) \leq x - y = 1 - \frac{2(1-\delta)}{e^{\epsilon} + 1}
\end{align*}
\end{proof}
\subsection{Lower Bound for $(\epsilon, \delta)$ DP : Proof of Theorem \ref{thm:dra-dp-lowerbound}}
\label{prf:dra-dp-lowerbound-proof}
\begin{proof}
Let $z_1, z_2 \in \cZ$ be the two farthest points in $\cZ$, i.e. $\rho(z_1, z_2) = \diam(\cZ)$. Note that by compactness of $\cZ$, $z_1$ and $z_2$ are guaranteed to exist, and since $\cZ$ is not a singleton, $z_1 \neq z_2$. Let $\cA$ be any arbitrary reconstruction adversary and let $\Delta = \tfrac{\diam(\cZ)}{2}$. For any $z \in \cZ$, let $P_z$ denote the output distribution of $\cM(\advData \cup \{z\})$ and let $P^{n}_{z}$ be its associated product measure (i.e. $P^{n}_z = \bigotimes_{i=1}^{n} P_z$) Note that this means $h_1, \dots h_n \iidsim P_z$ is equivalent to $h_{1:n} \sim P^n_z$ where we use $h_{1:n}$ as a shorthand for $h_1, \dots, h_n$. \\

\noindent By Markov's Inequality, the following holds for any $z \in \cZ$
\begin{align*}
    \bE_{h_1, \dots, h_n \iidsim \cM(\advData \cup \{z\})}[\rho(\cA(h_{1:n}), z)^2] &\geq \Delta^2 \bP_{h_1, \dots, h_n \iidsim \cM(\advData \cup \{z\})}[\rho(\cA(h_{1:n}), z) \geq \Delta] \\
    &= \Delta^2 P^{n}_z[\rho(\cA(h_{1:n}), z) \geq \Delta]
\end{align*}
Taking the supremum over $\cZ$ on both sides,
\begin{align}
\label{eq:lecam-bound-ineq}
    \sup_{z \in \cZ} \bE_{h_1, \dots, h_n \iidsim \cM(\advData \cup \{z\})}[\rho(\cA(h_{1:n}), z)^2] &\geq \Delta^2 \sup_{z \in \cZ} P^{n}_z[\rho(\cA(h_{1:n}), z) \geq \Delta)] \nonumber \\
    &\geq \tfrac{\Delta^2}{2} \left[P^{n}_{z_1}[\rho(\cA(h_{1:n}), z_1) \geq \Delta] + P^{n}_{z_2}[\rho(\cA(h_{1:n}), z_2) \geq \Delta]\right] 
\end{align}
Suppose the event $\rho(\cA(h_{1:n}), z_1) \geq \rho(\cA(h_{1:n}), z_2)$ were true. Then, the following would hold by an application of the triangle inequality : 
\begin{align*}
    \rho(\cA(h_{1:n}), z_1) &\geq \rho(\cA(h_{1:n}), z_2)  \\
    &\geq \rho(z_1, z_2) - \rho(\cA(h_{1:n}), z_1) \\
    &= 2 \Delta - \rho(\cA(h_{1:n}), z_1)
\end{align*}
Thus, we note that $\rho(\cA(h_{1:n}), z_1) \geq \rho(\cA(h_{1:n}), z_2) \implies \rho(\cA(h_{1:n}), z_1) \geq \Delta$ hence $P^{n}_1[\rho(\cA(h_{1:n}), z_1) \geq \Delta] \geq P^{n}_1[\rho(\cA(h_{1:n}), z_1) \geq \rho(\cA(h_{1:n}), z_2)]$. By parallel reasoning, $P^{n}_2[\rho(\cA(h_{1:n}), z_2) \geq \Delta] \geq P^{n}_2[\rho(\cA(h_{1:n}), z_2) \geq \rho(\cA(h_{1:n}), z_1)]$. Substituting into \eqref{eq:lecam-bound-ineq}, we obtain
\begin{align*}
    \sup_{z \in \cZ} \bE_{h_1, \dots, h_n \iidsim \cM(\advData \cup \{z\})}[\rho(\cA(h_{1:n}), z)^2] &\geq \tfrac{\Delta^2}{2} P^{n}_1[\rho(\cA(h_{1:n}), z_1) \geq \rho(\cA(h_{1:n}), z_2)] \\
    &+ \tfrac{\Delta^2}{2}P^{n}_2[\rho(\cA(h_{1:n}), z_1) \leq \rho(\cA(h_{1:n}), z_2)] \\
    &\geq \tfrac{\Delta^2}{2} P^{n}_1[\rho(\cA(h_{1:n}), z_1) \geq \rho(\cA(h_{1:n}), z_2)] \\
    &+ \tfrac{\Delta^2}{2}P^{n}_2[\rho(\cA(h_{1:n}), z_1) < \rho(\cA(h_{1:n}), z_2)] \\
    &\geq \tfrac{\Delta^2}{2} - \tfrac{\Delta^2}{2} P^{n}_1[\rho(\cA(h_{1:n}), z_1) < \rho(\cA(h_{1:n}), z_2)] \\
    &+ \tfrac{\Delta^2}{2}P^{n}_2[\rho(\cA(h_{1:n}), z_1) < \rho(\cA(h_{1:n}), z_2)] \\
    &\geq \tfrac{\Delta^2}{2}[1 - \TV(P^n_1, P^n_2)] \\
    &\geq \tfrac{\Delta^2}{2}[1 - \TV(P_1, P_2)]^n \\
    &= \tfrac{\Delta^2}{2} (1 - \TV(\cM(\advData \cup \{z_1\}), \cM(\advData \cup \{ z_2\})))^n \\
    &\geq \tfrac{\Delta^2}{2} \left(\frac{2(1-\delta)}{e^{\epsilon} + 1}\right)^n
\end{align*}
where the fifth inequality applies Lemma \ref{lem:tv-tensorization} and the last inequality applies Lemma \ref{lem:approx-dp-tv} using the fact that $\advData \cup \{z_1\}$ and $\advData \cup \{ z_2\}$ are datasets differing only in one point. Now, taking an infimum over all adversaries $\cA$ in the LHS, we obtain
\begin{align*}
    \cV(\cM, n) \geq \frac{\diam(\cZ)^2}{8} \left(\frac{2(1-\delta)}{e^{\epsilon} + 1}\right)^n
\end{align*}
By definition of $\cV(\cM, n)$ it is easy to see that if an adversary attains an expected squared reconstruction error of at most $\beta^2$ for every possible target point $z \in \cZ$, $\beta^2$ must be at least $\cV(\cM, n)$, i.e., the following must hold
\begin{align*}
    \beta^2 \geq \cV(\cM, n) \gtrsim \frac{\diam(\cZ)^2}{8} \left(\frac{2(1-\delta)}{e^{\epsilon} + 1}\right)^n
\end{align*}
Rearranging appropriately, we obtain the required query complexity lower bound. 
\end{proof}

\subsection{Proof of Lemma \ref{lem:approx-rdp-tv}}
\label{prf:lem-approx-rdp-tv}
\begin{proof}
The proof of this lemma is similar to that of Lemma \ref{lem:approx-dp-tv}. To this end, we consider any arbitrary $S \subseteq \cH$. By definition of total variation, the following holds
\begin{equation*}
\TV(\cM(\cD), \cM(\cD^{\prime}) \leq \bP[\cM(\cD) \in S] - \bP[\cM(\cD^{\prime}) \in S]
\end{equation*}
Note that since $\cM$ is an $(\alpha, \epsilon)$ Renyi DP mechanism and $\cD, \cD^{\prime}$ are neighboring datasets, the following constraints must be satisfied as per Lemma \ref{lem:rdp-set-bounds}:
\begin{align*}
    \bP[\cM(\cD) \in S] &\leq \left(e^{\epsilon} \bP[\cM(\cD^{\prime}) \in S]\right)^{\gamma} \\
    \bP[\cM(\cD^{\prime}) \in S] &\leq \left(e^{\epsilon} \bP[\cM(\cD) \in S]\right)^{\gamma}
\end{align*}
where $\gamma = \tfrac{\alpha - 1}{\alpha}$. Note that since $\alpha > 1, 0 < \beta < 1$. Subsequently, it is easy to see that $\TV(\cM(\cD), \cM(\cD^{\prime})$ can be upper bounded by the solution to the following optimization problem:
\begin{align}
\label{eq:tv-rdp-lp}
\ & \min x - y \nonumber \\
\text{subject to } & 0 \leq x \leq 1, \nonumber \\
\ & 0 \leq y \leq 1, \nonumber \\
\ & x \leq e^{\gamma\epsilon} y^{\beta}  \nonumber \\
\ & y \leq e^{\gamma\epsilon} x^{\beta} 
\end{align}
We now show that $x = \frac{e^{\gamma \epsilon}}{1 + e^{\gamma \epsilon}}$ and $y = \frac{1}{(1 + e^{\gamma \epsilon})^{\nicefrac{1}{\gamma}}}$ lie in the constraint set of the above problem. We first note that since $\epsilon > 0$ and $0 < \gamma < 1$, $0 \leq x \leq 1$ and $0 \leq y \leq 1$. Hence, the first two constraints are satisfied. Furthermore, $e^{\gamma \epsilon} y^{\gamma} = \tfrac{e^{\gamma \epsilon}}{1 + e^{\gamma \epsilon}} = x$, i.e., the third constraint is satisfied with equality. The final constraint can be verified by observing that:
\begin{align*}
    \frac{1}{(1 + e^{\gamma \epsilon})^{\nicefrac{1}{\gamma} - \gamma}} \leq 1 \leq e^{(\gamma + \gamma^2)\epsilon}
\end{align*}
where we use the fact that $\nicefrac{1}{\gamma} - \gamma \geq 0$ as $\gamma < 1$. Rearranging terms, it follows that:
\begin{align*}
    y = \frac{1}{(1 + e^{\gamma \epsilon})^{\nicefrac{1}{\gamma}}} \leq e^{\gamma \epsilon} \frac{e^{\gamma^2 \epsilon}}{(1 + e^{\gamma \epsilon})^{\gamma}} = e^{\gamma \epsilon} x^\beta 
\end{align*}
Hence, the final constraint is satisfied. Thus, the solution to the optimization problem \ref{eq:tv-rdp-lp} is bounded by $x-y$ and hence,
\begin{align*}
    \TV(\cM(\cD), \cM(\cD^{\prime})) \leq x - y \leq 1 - \frac{1}{1 + e^{\gamma \epsilon}} - \frac{1}{(1 + e^{\gamma \epsilon})^{\nicefrac{1}{\gamma}}}
\end{align*}
\end{proof}

%\subsection{$(0, \delta)$ DP : Proof of Corollary \ref{cor:dra-zerodelta-dp}}
%\label{prf:dra-zerodelta-dp-proof}
\subsection{$(\alpha, \epsilon)$ Renyi DP : Proof of Corollary \ref{cor:dra-renyi-dp}}
\label{prf:dra-renyi-dp-proof}
As suggested by the proof sketch in Section \ref{subsec:lb-sketch}, we follow the exact same steps as the proof of Theorem \ref{thm:dra-dp-lowerbound} in Section \ref{prf:dra-dp-lowerbound-proof} and obtain the following:
\begin{align*}
    \cV(\cM, n) &\geq \tfrac{\diam(\cZ)^2}{8} (1 - \TV(\cM(\advData \cup \{z_1\}), \cM(\advData \cup \{ z_2\})))^n 
\end{align*}
Substituting the TV upper bound obtained in Lemma \ref{lem:approx-rdp-tv}, we obtain the following:
\begin{align*}
    \cV(\cM, n) &\geq \tfrac{\diam(\cZ)^2}{8} \left[\frac{1}{e^{\gamma \epsilon} + 1} + \frac{1}{\left(e^{\gamma \epsilon} + 1\right)^{\nicefrac{1}{\gamma}}}\right]^n
\end{align*}
where $\gamma = 1-\nicefrac{1}{\alpha}$. Following the same arguments as Theorem \ref{thm:dra-dp-lowerbound}, the sample complexity is obtained by upper bounding $\cV(\cM, n)$ with $\beta^2$ and rearranging the expression to get a lower bound on $n$. 
\subsection{Reconstruction Upper Bound : Proof of Theorem \ref{thm:dra-dp-upperbound}}
\label{prf:dra-dp-upperbound-proof}
Since $\cZ = \{ z_1, z_2 \}$, $\diam(\cZ) = \rho(z_1, z_2)$. Let $z\in\cZ$ be the true sample, i.e., $\cD = \advData \cup \{ z \}$. Our construction of $\cM$ is a variant of the randomized response mechanism, defined as follows:
\begin{enumerate}
\item Flip a coin $C \sim \mathsf{Bernoulli}(\delta)$. If $C = 1$, then $\cM(\cD) := (1,z)$
\item Otherwise, if $C=0$, then $\cM(\cD)$ is defined as follows:
\begin{align*}
        \cM(\cD) :=
        \begin{cases}
    (0,z) \text{ with probability } \tfrac{e^\epsilon}{1+e^{\epsilon}}\\
    (0,\cZ\setminus\{z\}) \text{ with probability } \tfrac{1}{1+e^{\epsilon}}
    \end{cases}
\end{align*}
\end{enumerate}
\paragraph{Claim : $\cM$ is $(\epsilon, \delta)$ DP}Consider any $x, y \in \cZ$. Let $\cD_x = \advData \cup \{x\}$ and $\cD_y = \advData \cup \{y\}$. Let $C_x$ refer to the $\mathsf{Bernoulli}(\delta)$ coin flip in step 1 of $\cM(\cD_x)$ and $C_y$ denote the same for $\cM(\cD_y)$. Now, denote the events $E_x = \{ C_x = 0\}$ and $E_y = \{ C_y = 0\}$. Note that, $\bP[E_x] = \bP[E_y] = 1 - \delta$. Furthermore, when conditioned on $E_x$, $\cM(\cD_x)|_{E_x}$ is the randomized response mechanism and the same holds for $\cM(\cD_y)|_{E_y}$ conditioned on $E_y$. In particular, the following holds true for any $S \subseteq \{0,1 \} \times \cZ$:
\begin{align}
\label{eq:randomized-response-conditonal}    \bP[\cM(\cD_x) \in S | E_x] &\leq e^{\epsilon} \bP[\cM(\cD_y) \in S | E_y] \nonumber \\
 \bP[\cM(\cD_y) \in S | E_y] &\leq e^{\epsilon} \bP[\cM(\cD_x) \in S | E_x] 
\end{align}
Using the fact that for any two events $A, B$, $\bP[A|B]\bP[B] \leq \bP[A] = \bP[A|B]\bP[B] + \bP[A|B^c]\bP[B^c] \leq \bP[A|B]\bP[B] + \bP[B^c]$, we conclude the following
\begin{align*}
    \bP[\cM(\cD_x) \in S] &\leq \bP[\cM(\cD_x) \in S | E_x] \bP[E_x] + \bP[E^c_x] \\
    &= \bP[\cM(\cD_x) \in S | E_x] (1 - \delta) + \delta \\
    &\leq e^{\epsilon} \bP[\cM(\cD_y) \in S | E_y] (1 - \delta) + \delta \\
    &= e^{\epsilon} \bP[\cM(\cD_y) \in S | E_y] \bP[E_y] + \delta \\
    &\leq e^{\epsilon} \bP[\cM(\cD_y) \in S]  + \delta
\end{align*}
Repeating the same argument with $\cM(\cD_y)$, we infer that $\bP[\cM(\cD_y) \in S] \leq e^{\epsilon} \bP[\cM(\cD_x) \in S]  + \delta$. Hence, $\cM$ is $(\epsilon, \delta)$ differentially private. 
\paragraph{Reconstruction Adversary} Given $n$ i.i.d samples $h_i = (C_i, z_i) \iidsim \cM(\cD), \ i \in [n]$, the reconstruction adversary is defined as follows. Firstly, if there exists any $i \in [n]$ such that $C_i = 1$, $\cA(h_{1:n}) = z_i$. Under this event, our construction of $\cM$ ensures that the adversary does not incur any reconstruction error, i.e. $\cA(h_{1:n}) = z$. On the contrary, if $C_i = 0 \ \forall i \in [n]$, the adversary $\cA$ uses the following randomized strategy: It first constructs $n$ random bits $b_i$ defined as: 
\begin{align*}
        b_i :=
        \begin{cases}
            \sim \mathsf{Bern}(1-e^{-\e}) \text{ if } h_i = (0,z) \\
    0 \text{         if } h_i = (0,\cZ\setminus\{z\})\\
\end{cases}
\end{align*}
It then uses the random bits $b_i$ to output a reconstruction as follows:
\begin{align*}
        \cA :=
        \begin{cases}
    z \text{ if } \wedge_{i=1}^{n} b_i = 1\\
    \cZ\setminus\{z\} \text{ otherwise }\\
\end{cases}
\end{align*}
where $\wedge$ denotes the binary OR operator. In summary, conditioned on the event $\mathcal{C} = \{ C_i = 0 \ \forall \ i \in [n] \}$, the adversary returns the true data point $z$ if there exists any $i \in [n]$ such that $b_i = 1$, and returns $\cZ \setminus \{z\}$ otherwise. On the contrary, when conditioned on the event $\mathcal{C}^c = \{ \exists \ i \in [n] \ s.t. \ C_i = 1  \}$, the adversary outputs the true data point $z$ almost surely (by construction of $\cM$)\\

\noindent Note that, when conditioned on $\mathcal{C}$, $b_i \iidsim \mathsf{Bernoulli}(1-p)$ where $p = \tfrac{1}{1+e^{\epsilon}} + e^{-\epsilon} \cdot \tfrac{e^{\epsilon}}{1 + e^{\epsilon}} = \tfrac{2}{(1+e^{\epsilon})}$. Furthermore, by definition of $\cM$, $\bP[\mathcal{C}] = (1 - \delta)^n$. Now, define the event $\mathcal{B}$ as $\mathcal{B} = \{ b_i = 0 \ \forall \ i \in [n] \}$.  Clearly, $\bP[\mathcal{B} | \mathcal{C}] = p^n$. It follows that
\begin{align}
\label{eq:thm10-recon}
    \bE_{h_{1:n} \iidsim \cM(D)}[\rho(\cA(h_{1:n}), z)^2] &= \bE_{h_{1:n} \iidsim \cM(D)}[\rho(\cA(h_{1:n}), z)^2 | \mathcal{C}] \bP[\mathcal{C}] + \bE_{h_{1:n} \iidsim \cM(D)}[\rho(\cA(h_{1:n}), z)^2 | \mathcal{C}^c] \bP[\mathcal{C}^c] \nonumber \\
    &= (1-\delta)^n \bE_{h_{1:n} \iidsim \cM(D)}[\rho(\cA(h_{1:n}), z)^2 | \mathcal{C}] \nonumber \\
    &= (1-\delta)^n \bE_{h_{1:n} \iidsim \cM(D)}[\rho(\cA(h_{1:n}), z)^2 | \mathcal{B}, \mathcal{C}] \bP[\mathcal{B} | \mathcal{C}] \nonumber \\
    &+ (1-\delta)^n \bE_{h_{1:n} \iidsim \cM(D)}[\rho(\cA(h_{1:n}), z)^2 | \mathcal{B}^c, \mathcal{C}] \bP[\mathcal{B}^c | \mathcal{C}] \nonumber  \\
    &= (1 - \delta)^n \bE_{h_{1:n} \iidsim \cM(D)}[\rho(\cA(h_{1:n}), z)^2 | \mathcal{B}, \mathcal{C}] \bP[\mathcal{B} | \mathcal{C}] \nonumber \\
    &= (1-\delta)^n p^n \mathsf{diam}(\cZ)^2 \nonumber \\
    &= \left(\frac{2(1-\delta)}{e^{\epsilon} + 1}\right)^n \mathsf{diam}(\cZ)^2
\end{align}
where the second equality uses the fact that $\bP[\cA(h_{1:n}) = z | \mathcal{C}^c] = 1$, the fourth equality uses the fact that $\bP[\cA(h_{1:n}) = z | \mathcal{B}^c, \cC] = 1$ and the fifth equality uses the fact that $\bP[\cA(h_{1:n}) = \cZ \setminus \{z\} | \mathcal{B}, \cC] = 1$. \\

\noindent To make the RHS of Equation \eqref{eq:thm10-recon} equal to $\beta^2$, it suffices to set $n = \Theta\left(\tfrac{\ln(\nicefrac{\diam(\cZ)^2}{\beta^2})}{\ln\left(\tfrac{e^\epsilon + 1}{2(1-\delta)}\right)}\right)$. 

\paragraph{Validity of the Adversary $\cA$} We note that $\cA$ is a valid adversary in the set of all data reconstruction adversaries. In particular, it satisfies the desiderata put forward in Section \ref{sec:problem-formulation}, and its construction is consistent with the statement of Theorem \ref{thm:dra-dp-upperbound}, i.e., for every $z \in \cZ$ and $\beta \geq 0$ there exists a private learner $\cM$ and adversary $\cA$ that needs $\Theta\left(\tfrac{\ln(\nicefrac{\diam(\cZ)^2}{\beta^2})}{\ln\left(\tfrac{e^\epsilon + 1}{2(1-\delta)}\right)}\right)$ queries to $\cM$ for achieving a reconstruction error of $\beta^2$. In particular, $\cA$ takes $n$ i.i.d samples $h_{1:n} \sim \cM(\cD)$ and outputs a reconstruction $\cA(h_{1:n}) \in \cZ$ as per a randomized decision rule, which is completely in agreement with our modelling assumptions in Section \ref{sec:problem-formulation} (note that our threat model is not restricted to deterministic adversaries) and of prior work (\citet{Guo2022,Balle2022}). Finally, the fact that $\cA$ has prior knowledge of the values of the privacy parameters $(\epsilon, \delta)$ also does not lead to any inconsistency since the learner's privacy parameters $(\e,\de)$ are almost always publicly released (see Definition 1 of \citet{Dwork2011}).\\

\noindent Next, we construct a deterministic majority-based adversary and establish tightness of the same.

\paragraph{Tightness of Majority Adversary} Operating under the same assumptions and parameter settings as above, let $(\cZ, \rho)$ denote the two-point metric space, i.e., let $\cZ = \{z_1, z_2\}$, and let $z \in \cZ$, $\beta \geq 0$ be arbitrary. Moreover, let $\cM$ be as defined in the proof of Theorem \ref{thm:dra-dp-upperbound} and let $h_i = (C_i, z_i) \iidsim \cM(\cD), \ i \in [n]$. \\ 

\noindent We construct the \textit{majority adversary} $\tcA$ defined as follows: 
\begin{equation*} 
\begin{aligned}[c]
        \tcA(h_{1:n}) :=
        \begin{cases}
    z_i \text{ if } C_i = 1 \text{ for any }i\in[n]\\
    z \text{ if } \sum_{i\in[n]}\mathbb{I}\{h_i=(0,z)\}>\nicefrac{n}{2}\\
    \cZ\setminus\{z\} \text{ otherwise } 
\end{cases}
\end{aligned}
\end{equation*}
The following result establishes the optimality of this adversary for $\epsilon \geq 0.042$.

\begin{theorem*}[\textbf{Analysis of Majority Adversary}]
Let $(\cZ, \rho)$ denote the two-point metric space, i.e., let $\cZ = \{z_1, z_2\}$, and let $z \in \cZ$, $\beta \geq 0$ be arbitrary. Moreover, let $\cM$ be as defined in the proof of Theorem \ref{thm:dra-dp-upperbound} and $\tcA$ be as defined above. Then, for any $\epsilon \geq 0.042$ and $\delta \in [0,1]$, $\tcA$ achieves an expected squared reconstruction error of at most $\beta^2$ by making $\Theta\left(\tfrac{\ln(\nicefrac{\diam(\cZ)^2}{\beta^2})}{\ln\left(\tfrac{e^\epsilon + 1}{2(1-\delta)}\right)}\right)$ queries to $\cM$.
\end{theorem*}
\begin{proof}
Let $h_i = (C_i, z_i) \iidsim \cM(\cD), \ i \in [n]$. Let $\cC$ denote the event $\cC = \{ C_i = 0 \ \forall \ i \in [n] \}$. By definition of $\cM$, $\bP[\cC] = (1 - \delta)^n$ and $\bP[\tcA(h_{1:n}) = z | \cC^c] = 1$. It follows that:
\begin{align*}
    \bP[\tcA(h_{1:n}) = \cZ \setminus \{z\}] &= \bP[\tcA(h_{1:n}) = \cZ \setminus \{z\} | \cC] \bP[\cC] + \bP[\tcA(h_{1:n}) = \cZ \setminus \{z\} | \cC^c] \bP[\cC^c] \\
    &= (1-\delta)^n \ \bP[ \sum_{i\in[n]}\mathbb{I}\{h_i=(0,z)\}\leq\nicefrac{n}{2} | \cC]
\end{align*}
Note that, conditioned on $\cC$, $\mathbb{I}\{h_i=(0,z)\} \iidsim \mathsf{Bernoulli}(q)$ where $q = \tfrac{e^\epsilon}{e^\epsilon + 1} \geq \nicefrac{1}{2}$. Define $X = \sum_{i=1}^{n} \mathbb{I}\{h_i=(0,z)\}$. Clearly, conditioned on $\cC$, $X \sim \mathsf{Binomial}(n, q)$. Since $nq \geq \nicefrac{n}{2}$, we control $\bP[X \leq \nicefrac{n}{2}]$ via a Chernoff Bound \citep{WikiBinomail} to obtain:
\begin{align*}
    \bP[ X \leq \nicefrac{n}{2} | \cC] &\leq \exp\left(-\tfrac{n}{2} \cdot d_{\mathsf{KL}}(\nicefrac{1}{2}, q)\right) \\
    &= \exp(\tfrac{n}{2} \ln(4q(1-q))) \\
    &= [4q(1-q)]^{\nicefrac{n}{2}} \\
    &= (\tfrac{2}{e^{\epsilon} + 1})^n e^{\nicefrac{n\epsilon}{2}}
\end{align*}
where $d_{\mathsf{KL}}(\nicefrac{1}{2}, q) = \mathsf{KL}[\mathsf{Bernoulli}(\nicefrac{1}{2})||\mathsf{Bernoulli}(q)] = \tfrac{1}{2} \ln(\tfrac{1}{2q}) + \tfrac{1}{2} \ln(\tfrac{1}{2(1-q)}) = - \tfrac{1}{2}\ln(4q(1-q))$. It follows that,
\begin{align*}
    \bP[\tcA(h_{1:n}) = \cZ \setminus \{z\}] &\leq \left(\frac{2(1-\delta)}{e^{\epsilon} + 1}\right)^n e^{\nicefrac{n\epsilon}{2}}
\end{align*}
Thus, we conclude
\begin{align*}
    \bE_{h_{1:n} \iidsim \cM(D)}[\rho(\tcA(h_{1:n}), z)^2] &= \bP[\cA(h_{1:n}) = \cZ \setminus \{z\}] \mathsf{diam}(\cZ)^2 \\
    &\leq \left(\frac{2(1-\delta)}{e^{\epsilon} + 1}\right)^n e^{\nicefrac{n\epsilon}{2}} \mathsf{diam}(\cZ)^2
\end{align*}
To set the RHS of the above inequality equal to $\beta^2$, it suffices to set $n = \Theta\left(\frac{\ln(\nicefrac{\mathsf{diam}(\cZ)^2}{\beta^2})}{\ln(\frac{1+e^\e}{2(1-\de)})-\frac{\e}{2}}\right)$. \\

\noindent Now, consider $f : \bR \to \bR$ defined as $f(t) = 0.99 \ln(\tfrac{e^t + 1}{2}) - \nicefrac{t}{2}$. Note that $f(0.042) \geq 0$. Furthermore, $f^\prime(t) = 0.99 \tfrac{e^t}{1 + e^t} - 0.5 > 0 \ \forall t \geq 0.042$. Thus, $f(t) \geq 0 \ \forall t \ \geq 0.042$. Rearranging, we conclude that $\ln(\tfrac{e^t + 1}{2}) - \nicefrac{t}{2} \geq 0.01 \ln(\tfrac{e^t + 1}{2})$. \\ 

\noindent Hence, for any $\epsilon \geq 0.042$, $\ln(\tfrac{e^\epsilon + 1}{2}) - \nicefrac{\epsilon}{2} \geq 0.01 \ln(\tfrac{e^\epsilon + 1}{2})$ and for any $\delta \in [0, 1]$, $\ln(\tfrac{1}{1 - \de}) \geq 0$ and thus $\ln(\tfrac{1}{1 - \de}) \geq 0.01 \ln(\tfrac{1}{1-\delta})$. It follows that for any $\epsilon \geq 0.042$ and $\delta \in [0,1]$, 
\begin{align*}
    \ln(\tfrac{e^{\epsilon} + 1}{2(1-\delta)}) - \nicefrac{\epsilon}{2} \geq 0.01 \ln(\tfrac{e^{\epsilon} + 1}{2(1-\delta)}) \\
\end{align*}
Moreover, since $\ln(\tfrac{e^{\epsilon} + 1}{2(1-\delta)}) - \nicefrac{\epsilon}{2} \leq \ln(\tfrac{e^{\epsilon} + 1}{2(1-\delta)})$, we conclude the following
\begin{align*}
    \frac{1}{ \ln(\tfrac{e^{\epsilon} + 1}{2(1-\delta)})} \leq \frac{1}{\ln(\tfrac{e^{\epsilon} + 1}{2(1-\delta)}) - \nicefrac{\epsilon}{2}} \leq \frac{100}{ \ln(\tfrac{e^{\epsilon} + 1}{2(1-\delta)})} \ \forall \ \epsilon \geq 0.042, \delta \in [0, 1]
\end{align*}
Thus, for any $\epsilon \geq 0.042$ and $\delta \in [0,1]$, it suffices to set $n = \Theta\left(\tfrac{\ln(\nicefrac{\diam(\cZ)^2}{\beta^2})}{\ln\left(\tfrac{e^\epsilon + 1}{2(1-\delta)}\right)}\right)$ in order to obtain a squared reconstruction error of at most $\beta^2$,
\end{proof}

\subsection{Comparison with \cite{Guo2022}}
\label{app-subsec:guo-comparison}
For convenience, we restate the main result of \cite{Guo2022}
\begin{theorem}
[Theorem 1 of \cite{Guo2022}]
\label{app-thm:Guo-et-al}
    Let $(\cZ,\|\cdot\|_2)$ be the input metric space, with $\cZ\subseteq\reals^d$, of a $(2,\e)$-R\'enyi DP learner $\cM:\reals^d\to\cH$. Let $\cA$ be a reconstruction attack that outputs $\hat{z}(h)$ upon observing one sample from the learner's output distribution , i.e., $h\leftarrow\cM(\cD)$. Then, if $\zhat$ is unbiased,  
    \begin{equation}
        \bE[\|\hat{z(h)}-z\|^{2}_{2}] \geq \frac{\sum_{i=1}^{d}\mathsf{diam}_{i}(\cZ)^{2}}{4(e^{\e}-1)}
    \end{equation}
    where $\mathsf{diam}_{i}(\cZ)$ denotes the i-th coordinate-wise diameter defined as $ \mathsf{diam}_{i}(\cZ):=\underset{z,z'\in\cZ, z_{j} = z'_{j}\forall j\neq i}\sup|z_{i}-z_{i}'|$, and $C$ is a universal constant.
\end{theorem}
We shall now show that the above result is invalid for a large range of $\epsilon$. In particular, \textbf{Theorem 1 of \cite{Guo2022} yields invalid lower bounds for any $\epsilon < \ln(1 + \nicefrac{d}{4})$} \\

\noindent Consider a canonical instance of the data reconstruction problem where the data domain $\cZ$ is the unit ball in $(\reals^{d}, \|.\|_2)$. Then, $\sum_{i = 1}^{d} \mathsf{diam}_i(\mathcal{Z})^2 = 4d $ whereas $\mathsf{diam}(\mathcal{Z}) = 2$. Let $z \in \mathcal{Z}$ be arbitrary. As per our problem setup in Section \ref{sec:problem-formulation}, $\mathcal{M}(\mathcal{D}_{-}\cup\{ z\})$ denotes the output distribution induced by the randomized learner, when the target sample is $z$. By definition of the diameter, any data reconstruction attack $\hat{z}$ satisfies the trivial upper bound 

\begin{equation}\label{eq:expctn-upper}
    \bE_{h \sim \mathcal{M}(\mathcal{D}_{-}\cup\{ z\})}\left[\| \hat{z}(h) - z \|^2_{2}\right] \leq \mathsf{diam}(\mathcal{Z})^2 = 4
\end{equation}

\noindent However, the lower bound in Theorem 1 of \cite{Guo2022} states that, for any data reconstruction attack $\hat{z}$ on a $(2, \epsilon)$ Renyi Differentially Private learner, the reconstruction MSE is lower bounded as

$$\bE_{h \sim \mathcal{M}(\mathcal{D}_{-}\cup\{ z\})}\left[\| \hat{z}(h) - z \|^2_{2}\right] \geq \frac{\sum_{i = 1}^{d} \mathsf{diam}_{i}(\mathcal{Z})^2}{4(e^{\epsilon} - 1)} \geq \frac{d}{e^\epsilon - 1}$$
The RHS in the above bound is strictly greater than $\diam(\cZ)^2 = 4$ for any $\epsilon < \ln(1+\nicefrac{d}{4})$, thereby contradicting the trivial upper bound established above.

\section{Analysis of Data Reconstruction for $\epsilon$ Metric DP}
\label{app-sec:dra-mdp-proofs}
\subsection{Reconstruction Lower Bound : Proof of Theorem \ref{thm:dra-mdp-lowerbound}}
\label{prf:dra-mdp-lowerbound-proofs}
We borrow the notation of $P_z$ and $P^n_z$ from the proof of Theorem \ref{thm:dra-dp-lowerbound} in Appendix \ref{prf:dra-dp-lowerbound-proof}. Let $Z := \{z_1, \dots, z_M\} \subseteq \cZ$ be a finite subset of $\cZ$ (to be specified later) with $M > 1$. Following the same steps as the proof of Theorem \ref{thm:dra-dp-informal} by applying Markov's inequality, we obtain:

\begin{align}
\label{eq:fano-bound-ineq}
\sup_{z\in\cZ}\bE_{h_1, \dots, h_n \iidsim \cM(\advData \cup \{z\})}[\rho(\cA(h_{1:n}), z)^2] &\geq \Delta^2 \sup_{z \in \cZ} P^{n}_z[\rho(\cA(h_{1:n}), z) \geq \Delta)] \nonumber \\
    &\geq \Delta^2 \sup_{j \in [M]}P^{n}_{z_j}[\rho(\cA(h_{1:n}), z_j) \geq \Delta)]
\end{align}
where the last inequality follows from the fact that $Z \subseteq \cZ$. For a tight lower bound, we must carefully construct the set $Z$. We recall from Section \ref{sec:proof-sketch} that the adversary's task of reconstructing a training sample $z\in\cZ$ is at least as hard deciding which among the $\{z_1,\dots z_M\}$ was included in the training set, i.e., reconstruction is at least as hard as an $M$-ary hypothesis testing problem. To reduce from reconstruction to testing, we construct a sample space $\{z_1,\dots, z_M\}$ such that the elements are uniformly distinguishable i.e., $\rho(z_i,z_j)\geq 2\Delta \ \forall i\neq j\in[M]$, i.e., $Z$ is a $2\Delta$ packing in the $\rho$ metric. For ease of exposition, we assume $\rho$ is a homogeneous metric. 

\noindent Given this construction, we again follow the same steps as the proof of Theorem \ref{thm:dra-dp-lowerbound}. In particular, suppose there exists $z_j, z_k$ such that $\rho(\cA(h_{1:n}),z_j)\geq\rho(\cA(h_{1:n}),z_k)$ were true, then by an application of triangle inequality, one would have:
        \begin{align*}
            \rho(\cA(h_{1:n}),z_j)&\geq\rho(\cA(h_{1:n}),z_k)\\
            &\geq\rho(z_k,z_j) - \rho(\cA(h_{1:n}),z_j)\\
            &\geq2\Delta - \rho(\cA(h_{1:n}),z_j)
        \end{align*}
Thus, $\rho(\cA(h_{1:n}),z_j)\geq\rho(\cA(h_{1:n}),z_k) \implies \rho(\cA(h_{1:n}),z_j)\geq\Delta$. Defining the minimum distance test $\Psi$ as $\Psi := \underset{j\in [M]}{\operatorname{argmin}}\rho(\cA(h_{1:n}), z_j)$, we conclude that
\begin{align*}
        P^{n}_{z_j}[\rho(\cA(h_{1:n}), z_j)\geq \Delta)]&\geq P^{n}_{z_j}[\Psi(h_{1:n}))\neq z_j]
\end{align*}
It follows that
\begin{align}
\label{eq:fano-mid-step}
\sup_{z\in\cZ}\bE_{h_1, \dots, h_n \iidsim \cM(\advData \cup \{z\})}[\rho(\cA(h_{1:n}), z)^2] &\geq \Delta^2 \sup_{z \in \cZ} P^{n}_z[\rho(\cA(h_{1:n}), z) \geq \Delta)] \nonumber \\
    &\geq \Delta^2 \sup_{j \in [M]}P^{n}_{z_j}[\Psi(h_{1:n}))\neq z_j] \nonumber \\
    &\geq \Delta^2\left[1-\frac{1/M^{2}\sum_{j,k=1}^{M}\mathsf{KL}(P_{z_j}^n,P_{z_k}^n) + \log 2}{\log(M-1)}\right] \nonumber \\
    &= \Delta^2\left[1-\frac{\nicefrac{n}{M^{2}}\sum_{j,k=1}^{M}\mathsf{KL}(P_{z_j},P_{z_k}) + \log 2}{\log(M-1)}\right]
\end{align}
where the second inequality applies Lemma \ref{lemma:fanos} and the last inequality uses the tensorization of KL divergence for product measures.  \\

\noindent Now, let $S$ be maximal $\nicefrac{1}{2}$-packing of the unit ball $\mathbb{B}(\cZ)$. Note that $2 \geq |S| < \infty$ due to the local compactness of $\cZ$ \citep{rudin1976principles}. Since $\cZ$ is a linear homogeneous metric space, we can scale $S$ by a factor of $4 \Delta$ (assuming $\Delta \leq \nicefrac{1}{4}$ without loss of generality) to obtain a maximal $2 \Delta$ packing of $\mathbb{B}(\cZ)$. Let this scaling of $S$ be the set $Z$. Note that $M = |Z| = |S|$ and by maximality of the packing $S$,  $Z$ is also a $2 \Delta$ covering of $\mathbb{B}(\cZ)$. Thus, $\rho(z_j, z_k) \leq 2 \Delta$ and thus, by Lemma \ref{lem:kl-upper-bounds}, the following holds:
\begin{align*}
    \mathsf{KL}(P_{z_j} || P_{z_k}) = \mathsf{KL}(\cM(\advData \cup \{ z_j \}) || \cM(\advData \cup \{ z_k \})) \leq \frac{\epsilon^2 \rho(z_j, z_k)^2}{2} \leq 2 \epsilon^2 \Delta^2
\end{align*}

\begin{align*}
    \bE_{h_1, \dots, h_n \iidsim \cM(\advData \cup \{z\})}[\rho(\cA(h_{1:n}), z)^2] &\gtrsim \Delta^2\left[1-\frac{n\e^{2}\Delta^{2} + \log 2}{\ln M(\nicefrac{1}{2}, \mathbb{B}(\cZ), \rho)}\right]
\end{align*}
Optimizing over $\Delta$ and taking an infimum over all adversaries $\cA$ gives us $\cV(\cM, n) \gtrsim \tfrac{\Tilde{d}}{n\epsilon^2}$ where $\Tilde{d} = \ln M(\nicefrac{1}{2}, \mathbb{B}(\cZ), \rho)$. As before, the required query complexity is obtained by upper bounding $\cV(\cM, n)$ with $\beta^2$ and rearranging the inequality to obtain a lower bound on $n$.  
\subsection{Reconstruction Upper Bound : Proof of Theorem \ref{thm:dra-mdp-upperbound}}
\label{prf:dra-mdp-upperbound-proofs}
Our upper bound construction is based upon the exponential mechanism in $(\bR^d, \|\dot\|_2)$. In particular, for $\mu \in \bR^d$, let $\pi_{\mu}$ be a probability measure on $\bR^d$ whose density with respect to the Lebesgue measure $\mathsf{Leb}$ is given by,
\begin{align*}
    \frac{\mathsf{d} \pi_{\mu, \epsilon}}{\mathsf{d} \mathsf{Leb}}(x) \propto e^{-\epsilon \|x - \mu\|_2}
\end{align*}
For any $z \in \bR^d$, we set the output distribution of $\cM$ to be $\cM(\cD) = \cM(\advData \cup \{z\}) = \pi_{z, \epsilon}$. Note that for any $\mu_1, \mu_2 \in \bR^d$, their ratio of densities is bounded as follows due to the triangle inequality
\begin{align*}
    \frac{\mathsf{d} \pi_{\mu_1, \epsilon}}{\mathsf{d} \pi_{\mu_2, \epsilon}}(x) = e^{\epsilon(\|x - \mu_2\|_2 - \|x - \mu_1 \|_2)} \leq e^{\epsilon \|\mu_1 - \mu_2\|_2}
\end{align*}
Note that the same upper bound holds for $\frac{\mathsf{d} \pi_{\mu_2, \epsilon}}{\mathsf{d} \pi_{\mu_1, \epsilon}}$ It follows that for any Borel set $A \subseteq \bR^d$, $$e^{-\epsilon \|\mu_1 - \mu_2\|_2}\pi_{\mu_2, \epsilon}(A) \leq \pi_{\mu_1, \epsilon}(A) \leq e^{\epsilon \|\mu_1 - \mu_2\|_2} \pi_{\mu_2, \epsilon}(A)$$
Consequently, $\cM(\cD)$ is $\epsilon$ Lipchitz DP w.r.t $\|\cdot\|_2$. \\

We consider a reconstruction adversary which draws $h_1, \dots, h_n \iidsim \cM(\cD)$ and computes $\cA(h_1, \dots, h_n) = \tfrac{1}{n} \sum_{i=1}^{n} h_i$. It is easy to see that $\bE_{h_1, \dots, h_n \iidsim \cM(\cD)}[\|\cA(h_1, \dots, h_n) - z\|^{2}_{2}] = \frac{\sigma^2}{n}$ where $\sigma^2 = \bE_{x \sim \pi_{z, \epsilon}}[\|x - z\|^{2}_{2}]$.

To sharply bound $\sigma^2$, we first note that by Theorem 1 of \cite{bobkov2003spectral} shows that for any $\mu \in \bR^d$, the Poincare constant (or the inverse spectral gap) of $\pi_{\mu, \epsilon}$ is $\lambda = \Theta(\nicefrac{d}{\epsilon^2})$. We now bound $\sigma^2$ using the fact that for any distribution, an inverse spectral gap of $\lambda$ implies subexponential concentration with parameter $\Theta(\sqrt{\lambda})$ \citep{bobkov1997poincare, aida1994moment}. In particular, Theorem 2.7 of \cite{huang2021poincare} ensures that $\sigma^2 \leq \lambda \log^2(d) = \tfrac{d \log^2(d)}{\epsilon^2}$. It follows that,
\begin{align*}
    \bE_{h_1, \dots, h_n \iidsim \cM(\cD)}[\|\cA(h_1, \dots, h_n) - z\|^{2}_{2}] \lesssim \frac{d \log^2(d)}{n \epsilon^2}
\end{align*}
To make the RHS at most $\beta^2$, it suffices to set $n = \Theta(\tfrac{d \log^2(d)}{\epsilon^2 \beta^2})$

\section{Metric Differential Privacy in Practice}
\label{app-sec:mdp-in-practice}
We first note that mDP is a strict generalization of DP.
\begin{lemma}[Metric DP is a generalization of DP]
\label{lem:mdp-v-dp}
Let $\cZ$ be an arbitrary collection of secrets. A randomized learner $\cM:\cZ\to\cH$ is $\e$-differentially private if and only if $\cM$ is $\e$-metric differentially private with respect to the Hamming metric on $\cZ^N$. 
\end{lemma}
\begin{proof}
    See Lemma 2.3 of \cite{boedihardjo2022private}
\end{proof}

Next, we establish the requisite framework for analyzing mDP in practice and present an extension of the Gaussian Noise Mechanism to metric privacy. We emphasize that, since mDP is a generalization of DP, practical applications would require a relaxation similar to \ref{defn:rdp} in order to establish privacy guarantees. To this end, we consider the following relaxed version of metric differential privacy. 
%\begin{definition}[Approximate mDP \cite{Imola2022}]
%\label{def:approximate-lipschitz-privacy}
%Given a metric space \((\dataspace, \rho)\), for any \(\epsilon_{L}\geq 0, \delta\geq 0\), a randomized learner \(\learner: \dataspace \rightarrow\labelspace\) is \((\epsilon_{L}, \delta)-\)mDP if for every pair of input datasets \(\dataset, \dataset'\in\dataspace\) and \(\forall T\subseteq\labelspace\),
%\begin{equation}
%    \prob[\learner(\dataset)\in T]\leq e^{\epsilon_{L}\rho(\dataset, \dataset')}\prob[\learner(\dataset')\in T] + \delta
%\end{equation}
%\end{definition}
\begin{definition}[R\'enyi mDP]
\label{def:renyi-lipschitz}
Given a metric space \((\cZ, \rho)\), for any \(\epsilon\geq 0, \alpha\in[1,\infty]\), a randomized learner \(\cM: \cZ^n \to\cH\) is \((\alpha, \epsilon)-\)R\'enyi mDP if for every pair of input datasets \(\cD, \cD'\in\cZ\) the R\'enyi Divergence is bounded as follows
    \begin{equation}
        D_{\alpha}(P||Q) = \frac{1}{\alpha - 1}\log\bE_{x\sim Q}\left[\left(\frac{P(x)}{Q(x)}\right)^{\alpha} \right]\leq\epsilon\rho(\cD,\cD')
    \end{equation}
    where \(P\) and \(Q\) denote the output distributions \(\cM(\cD)\) and \(\cM(\cD')\), respectively. 
\end{definition}
\noindent It is easy to see that Definition \ref{def:renyi-lipschitz} recovers R\'enyi DP when restricted to the Hamming metric, in a manner similar to how mDP generalizes pure DP. Before proceeding further, we recall standard notions of sensitivity and noise mechanisms for classical differential privacy.

\noindent A common paradigm in releasing statistical information with differential privacy is to generate a noisy estimate of the true statistic. Namely, if $f: \cZ\to\reals^{d}$ is a real-valued function\footnote{The restriction to real-valued functions is not essential.}, $\cM(\cD) = f(\cD) + \eta$ is differentially private for an appropriate noise level $\eta$. The magnitude of perturbation is generally calibrated according to the sensitivity of the function $f$, defined as follows:
%We also note that mechanisms in differential privacy generally require the following sensitivity constraint:
\begin{definition}[Sensitivity]
\label{defn:global-sensitivity}
    Given any function \(f:\cZ\to\reals^{d}\),
 \begin{equation}
     \Delta_{f} = \underset{\cD,\cD'\in\cZ ; \|\cD-\cD'\|\leq 1}{\max}\|f(\cD)-f(\cD')\|
 \end{equation}
\end{definition}
%Note that the sensitivity constraint is over all neighboring datasets \(\dataset, \dataset'\in\dataspace\).
The Gaussian mechanism, for instance, requires that $f(\cD)$ be perturbed with noise drawn from the Gaussian distribution, as follows:
\begin{definition}[Gaussian Mechanism]
\label{defn:dp-gaussian-mech}
    Given any function \(f:\cZ\to\reals^{d}\), the Gaussian mechanism is defined by
 \begin{equation}
     \cM(\cD) = f(\cD) + \mathcal{N}(0, \Delta_{f}^{2}\sigma^{2})
 \end{equation}
 For $\epsilon < 1$ and $c^{2} > 2\ln(1.25/\delta)$, the Gaussian mechanism with parameter $\sigma \geq c\Delta_{f}/\epsilon$ satisfies $(\epsilon, \delta)$-DP \citep{Dwork2017}.
\end{definition}

The sensitivity constraint described in Definition \ref{defn:global-sensitivity} is over all neighboring datasets \(\cD, \cD'\in\cZ\), in order to accommodate for classical differential privacy. For metric differential privacy, we define a relaxed requirement which we call input lipschitzness. 
\begin{definition}[Input Lipschitzness]
 A function \(f:\cZ\to\reals^{d}\) is $L_{input}$ input lipschitz if the following holds
    \begin{equation}
        \|f(\cD)-f(\cD')\|\leq L_{input}\rho(\cD,\cD')
    \end{equation}
\end{definition}
%Having set up the framework for analyzing mDP in practice, we now investigate the feasibility of extending popular mechanisms in differential privacy to accommodate for Lipschitz privacy. 
Equipped with this notion, we extend the Gaussian mechanism for DP (\ref{defn:dp-gaussian-mech}) to mDP, as follows:
\begin{proposition}[Gaussian Mechanism for mDP]
\label{prop:gaussian-lp}
    Consider any two datasets \(\cD, \cD' \in \cZ\), at a distance \(\rho(\cD,\cD')\) and an input lipschitz function \(f:\cZ\to\reals^{d}\). For any \(\delta \in (0,1)\) and \(c^{2}>\ln(1.25/\delta)\), the Gaussian Mechanism with parameter \(\sigma \geq cL_{input}/\epsilon\) is \((\epsilon,\delta)-\)mDP, where  \(L_{input}\) is the input Lipchitz constant of \(f\).
\end{proposition}
\begin{proof}
From the definition of input Lipschitzness and sensitivity, it follows that 
\begin{equation}
    \Delta_{f} = \|f(\cD)-f(\cD')\|\leq L_{input}\rho(\cD,\cD')
\end{equation}
 Setting \(\epsilon = \epsilon\rho(\cD, \cD')\), we are required to bound the following quantity for \(f\) applied on datasets \(\cD, \cD'\):

 \begin{equation}
\left|\ln\frac{e^{\frac{-1}{2\sigma^{2}}\|x\|^{2}}}{e^{\frac{-1}{2\sigma^{2}}\|x + \Delta_{f}\|^{2}}}\right|\leq \epsilon = \epsilon\rho(\cD, \cD')
 \end{equation}
 where the numerator is the probability of observing a specific output \(f(\cD)+x\) and the denominator is the probability of observing the same output when the input dataset is replaced by \(\cD'\). The proof of our claim then follows trivially from the proof of Gaussian Mechanism for \((\epsilon, \delta)\) differential privacy \citep{Dwork2017}. 
\end{proof}

We are now ready to develop mDP accountants for privacy-preserving algorithms. A popular approach to differentially private deep learning involves an extension of Stochastic Gradient Descent (SGD)\citep{Abadi2016} to incorporate gradient clipping and noisy gradient updates. The resulting algorithm is called Differentially Private SGD (DP-SGD), wherein, $(\epsilon, \delta)$ differential privacy is guaranteed via an application of the Gaussian mechanism at each iteration. Since mDP is a generalization of traditional DP, we expect that iterative DP algorithms that rely on additive noise at each intermediate solution step are often inherently mDP, owing to the application of composition theorems that admit natural extensions to arbitrary metric spaces.

\subsection{Metric DP Accounting for DP-SGD}
In the analysis of DP-SGD \citep{Abadi2016}, privacy is attained via the addition of Gaussian distributed noise. Standard arguments on the composition of privacy mechanisms renders the differentially private variant of SGD \((q\epsilon, q\delta)-\)differentially private at each step, where \(q = L/N\) is the lot size. We note that, by replacing the global sensitivity assumption with input Lipschitzness, the Gaussian mechanism ensures \((\epsilon, \delta)\)-mDP for appropriately scaled noise, as in Proposition \ref{prop:gaussian-lp}. Thus, DP-SGD inherently incorporates metric differential privacy, with privacy accounting based on standard composition theorems.

\subsection{Privacy Analysis for Metric DP in PN-SGD}\label{appendix:propositionPNSGD}
When analysing the privacy of learning algorithms that require iterative updates on an intermediate solution, it is common practice to ensure privacy at each iteration and argue about the cumulative loss of privacy via composition theorems. Another popular direction is the theoretical analysis of Noisy Stochastic Gradient Descent to formalize privacy amplifications under certain assumptions and obtain bounds on the degradation of privacy across iterations. We extend the analysis of one such algorithm, the Projected Noisy Stochastic Gradient Descent (PNSGD) \citep{Feldman2018}, restated in Algorithm \ref{alg:pnsgd}, and establish that the algorithm satisfies $(\alpha, \epsilon)$-R\'enyi mDP with a lower noise magnitude than that required for $(\alpha, \epsilon)$-R\'enyi DP.
\begin{algorithm}[H]
   \caption{Projected Noisy Stochastic Gradient Descent (Algorithm 1 of \cite{Feldman2018})}
   \label{alg:pnsgd}
\begin{algorithmic}
   \STATE {\bfseries Input:} Dataset $S = \{x_{1}, ..., x_{n}\},  f: \mathcal{K}\times\mathcal{X}\rightarrow \reals$, which is a convex function in the first parameter, learning rate $\eta$, starting point $w_{0}\in\mathcal{K}$, noise parameter $\sigma$
   \FOR{$t\in\{0, ..., n-1\}$}
   \STATE  $v_{t+1}\leftarrow w_{t} - \eta(\nabla_{w}f(w_{t}, x_{t+1})+Z)$, where $Z\sim\mathcal{N}(0,\sigma^{2}\mathbb{I}_{d})$.
   \STATE  $w_{t+1}\leftarrow\prod_{\mathcal{K}}(v_{t+1})$, where $\prod_{\mathcal{K}}(w)=\arg\min_{\theta\in\mathcal{K}}\|\theta-w\|_{2}$ is the $l_{2}$ projection on $\mathcal{K}$
   %\FOR{$i=1$ {\bfseries to} $m-1$}
   %\IF{$x_i > x_{i+1}$}
   %\STATE Swap $x_i$ and $x_{i+1}$
   \ENDFOR
   \STATE return the final iterate $w_{n}$ 
   %\ENDIF
   %\ENDFOR
\end{algorithmic}
\end{algorithm}
\begin{proposition*}
\label{app:prop-pnsgd}
    Let $\cK \subset \mathbb{R}^d$ be a convex set and let $\{ f(., x) \}_{x \in cX}$ be a family of convex, $G$ Lipschitz, $\beta$-smooth functions over $\cK$, where the gradients are $L_{input}$-input Lipschitz. Furthermore, assume $\cX$ is a bounded set. Then, for any $\eta \leq 2/\beta$ and $\alpha > 1$, initializing $w_0 \in \cK$ and dataset $S \in \cX^{n}$, PNSGD run with $\sigma^{2} \geq \frac{\alpha G L_{input}}{\epsilon (n- t + 1)}$ satisfies $(\alpha, \epsilon)$-R\'enyi mDP.
\end{proposition*}
\begin{proof}
    The proof of this theorem is an extension of Theorem 23 of \cite{Feldman2018} to Renyi mDP. Let $S \coloneqq (x_1, \dots, x_n)$ and $S' \coloneqq (x_1, \dots, x_{t-1}, x^{'}_{t}, \dots, x_n)$ be two arbitrary datasets that differ in index $t$. Since $\eta \leq 2/\beta$, the projected SGD updates are contractive noisy iterations as per Definition 19 of \cite{Feldman2018}. 

We define the projected SGD operators for the dataset $S$ as $g_i(w) = \Pi_{\cK}(w) - \eta \nabla f(\Pi_{\cK}(w), x_i) , \ i \in [n]$. The projected SGD operators $g^{'}_i$ for the dataset $S'$ are defined in a similar fashion. Since the datsets differ only in the $t^{\textrm{th}}$ index, we note that $g_i = g^{'}_i \ \forall i \in [n] \setminus \{ t \}$ and,
\begin{align*}
    \sup_{w} \| g_{t}(w) - g^{'}_{t}(w) \|_{2} &= \eta \sup_{w} \| \nabla f(\Pi_{\cK}(w), x_t) - \nabla f(\Pi_{\cK}(w), x^{'}_t) \|_{2}  \\
    &= \eta \sup_{w} \sqrt{\| \nabla f(\Pi_{\cK}(w), x_t) - \nabla f(\Pi_{\cK}(w), x^{'}_t) \|_{2}} \sqrt{\| \nabla f(\Pi_{\cK}(w), x_t) - \nabla f(\Pi_{\cK}(w), x^{'}_t) \|_{2}} \\
    &\leq \eta \sqrt{2 G L_{input} \| x_t - x^{'}_t \|_{2}}
\end{align*}
Now, applying Theorem 22 of \cite{Feldman2018} with $a_1, \dots, a_{t-1} = 0$, $a_t, \dots, a_n = \frac{\eta \sqrt{2 G L_{input} \| x_t - x^{'}_t \|_{2}}}{n-t+1}$, $s_t = \eta \sqrt{2 G L_{input} \| x_t - x^{'}_t \|_{2}}$ and $s_i = 0, \ \forall i \in [n] \setminus \{ t \}$, we conclude,
\begin{align*}
    D_{\alpha}(w_n || w^{'}_n) \leq \frac{\alpha}{2 \eta^2 \sigma^2} \sum_{i=1}^{n} a^2_i \leq \frac{ \alpha G L_{input} \| x_t - x^{'}_t \|_{2}}{\sigma^2 (n-t+1)}
\end{align*} 
From the above inequality, we conclude that setting $\sigma^2 \geq \frac{\alpha G  L}{\epsilon (n- t + 1)}$ suffices to ensure $D_{\alpha}(w_n || w^{'}_n) \leq \epsilon \| x_t - x^{'}_t \|$.
\end{proof}
\section{Related Literature}
\label{sec:related-literature}
Theoretical investigation of reconstruction adversaries in the framework of differential private machine learning was initially studied in the recent works of \cite{Balle2022} and \cite{Guo2022}, in the bayesian and frequentist settings respectively. Their analyses, however, only considers a single round of communication between the adversary and the learner, thereby limiting the expressibility of the attack. \cite{Balle2022} introduced the notion of \emph{reconstruction robustness} to bound the success probability of an attack against $\e$DP learners, and, under certain modeling assumptions on the adversary's prior belief, derived upper bounds as a function of $\e$. This metric is somewhat restrictive, since it is only equivalent to DP when considering \emph{perfect} reconstruction, and requires a very specific class of priors, concentrated on pairs of data points(see Theorem 5 in \cite{Balle2022}). Instead, we focus on the minimax rate of reconstruction error of an adversary, described in detail in Section \ref{sec:problem-formulation}. This is more general than reconstruction robustness, and captures a privacy game between the learner and the adversary. Informally, the learner chooses a target sample that is the hardest to reconstruct, while the adversary picks the optimal attack that minimizes the error of reconstruction.

\noindent \cite{Guo2022} operate in the frequentist setting, and obtain lower bounds on the Euclidean MSE of reconstruction through the output of $(2,\e)$ R\'enyi DP learners. A key limitation of their work is that their lower bound result is invalid for a broad range of $\e$, as we formally state in Section \ref{sec:dra-dp-bounds}, with corresponding proof in Appendix \ref{app-subsec:guo-comparison}. Furthermore, prior work requires the input metric space to be compact and relies on the existence of coordinate-wise diameters. Such stringent assumptions hinder the generalizability of their results beyond $\reals^{d}$, particularly in unstructured domains such as text, where the notion of coordinate-wise diameters may not be well-defined or relevant.

\noindent \citet{stock2022defending} obtain lower bounds on the leakage of secrets in language tasks in R\'enyi differentially private learners. However, they are restricted to a structured and simplistic notion of secret bits, which does not capture the complexity of privacy protection in unstructured textual data. On the contrary, our lower bound results are more general and hold for arbitrary metric spaces.
\end{document}